\documentclass{article}

    \usepackage[preprint]{neurips_2024}

\usepackage[utf8]{inputenc} % allow utf-8 input
\usepackage[T1]{fontenc}    % use 8-bit T1 fonts
\usepackage{hyperref}       % hyperlinks
\usepackage{url}            % simple URL typesetting
\usepackage{booktabs}       % professional-quality tables
\usepackage{amsfonts}       % blackboard math symbols
\usepackage{nicefrac}       % compact symbols for 1/2, etc.
\usepackage{microtype}      % microtypography
\usepackage{xcolor}         % colors

\input{Defs/mathdefs}
\usepackage{amsthm}
\usepackage{amssymb}
\usepackage{thmtools}

\declaretheoremstyle[
	    spaceabove=\topsep, 
	    spacebelow=\topsep, 
	    headfont=\normalfont\bfseries,
	    bodyfont=\normalfont\itshape,
	    notefont=\normalfont\bfseries,
	    notebraces={(}{)},
	    postheadspace=0.33em, 
	    headpunct={.},
    ]{theorem}
\declaretheorem[style=theorem]{theorem}

\declaretheoremstyle[
	    spaceabove=\topsep, 
	    spacebelow=\topsep, 
	    headfont=\normalfont\bfseries,
	    bodyfont=\normalfont,
	    notefont=\normalfont\bfseries,
	    notebraces={(}{)},
	    postheadspace=0.33em, 
	    headpunct={.},
    ]{definition}

\declaretheoremstyle[
        spaceabove=\topsep, 
        spacebelow=\topsep, 
        headfont=\normalfont\bfseries,
        bodyfont=\normalfont,
        notefont=\normalfont\bfseries,
        notebraces={}{},
        postheadspace=0.33em, 
        qed=$\blacksquare$, 
        headpunct={.},
    ]{proofstyle}
\declaretheorem[style=proofstyle,numbered=no,name=Proof]{proof}

\declaretheorem[style=theorem,sibling=theorem,name=Lemma]{lemma}

\declaretheorem[style=theorem,numbered=no,name=Theorem]{theorem*}
\declaretheorem[style=theorem,numbered=no,name=Lemma]{lemma*}
\declaretheorem[style=theorem,numbered=no,name=Corollary]{corollary*}
\declaretheorem[style=theorem,numbered=no,name=Proposition]{proposition*}
\declaretheorem[style=theorem,numbered=no,name=Claim]{claim*}
\declaretheorem[style=theorem,numbered=no,name=Fact]{fact*}
\declaretheorem[style=theorem,numbered=no,name=Observation]{observation*}
\declaretheorem[style=theorem,numbered=no,name=Conjecture]{conjecture*}

\declaretheorem[style=definition,numbered=no,name=Definition]{definition*}
\declaretheorem[style=definition,numbered=no,name=Remark]{remark*}
\declaretheorem[style=definition,numbered=no,name=Example]{example*}
\declaretheorem[style=definition,numbered=no,name=Question]{question*}

\newcommand{\piClass}{\Pi_M}

\newcommand{\pBarKe}[1][\kEpoch]{\bar{P}^{#1}}
\newcommand{\phiBarH}[1][h]{\bar{\phi}_h}

\newcommand{\Egood}{E_g}
\newcommand{\EgoodBar}{\bar{E}_g}

\newcommand{\setEpochs}{[E]}
\newcommand{\numEpochs}{E}
\newcommand{\setEpochsE}{K_e}
\newcommand{\kEpoch}{k_e}
\newcommand{\kPrevEpoch}{k_{e-1}}

\newcommand{\covK}[1][k]{{\Lambda^{#1}}}

\newcommand{\covKH}[1][k]{{\Lambda^{#1}_{h}}}

\newcommand{\hatCovK}[1][k]{{\covK[\kEpoch]}}
\newcommand{\hatCovKH}[1][h]{{\Lambda^{\kEpoch}_{#1}}}

\newcommand{\betaR}{\beta_r}
\newcommand{\betaP}{\beta_p}
\newcommand{\betaPh}[1][h]{\beta_{p,#1}}

\newcommand{\betaQ}{\beta_Q}
\newcommand{\betaQh}[1][h]{\beta_{Q,#1}}
\newcommand{\betaWarmup}{\beta_w}

\newcommand{\M}{\mathcal{M}}
\newcommand{\A}{\mathcal{A}}
\newcommand{\X}{\mathcal{X}}

\newcommand{\Q}{\mathcal{Q}}
\newcommand{\V}{\mathcal{V}}

\newcommand{\etaO}{\eta_o}

\newcommand{\hatPsiH}[1][h]{\widehat{\psi}_{#1}}
\newcommand{\psiKH}[1][h]{\widehat{\psi}_{#1}^{k}}

\newcommand{\psiH}[1][h]{\psi_{#1}}

\newcommand{\thetaH}[1][h]{\theta_{#1}}
\newcommand{\thetaKH}[1][h]{\theta^{k}_{#1}}
\newcommand{\hatThetaK}[1][k]{{\widehat{\theta}^{#1}}}
\newcommand{\hatThetaKH}[1][h]{{\widehat{\theta}^{k}_{#1}}}

\newcommand{\piK}[1][]{\pi^k_{#1}}
\newcommand{\piOpt}[1][]{\pi^\star_{#1}}

\newcommand{\regret}{\mathrm{Regret}}

\newcommand{\sig}{\sigma}

\newcommand{\betaB}{\beta_b}
\usepackage{cleveref}
\usepackage{algorithm}
\usepackage{algorithmic}

\Crefname{ALC@unique}{Line}{Lines}

\usepackage[textsize=tiny]{todonotes}

\title{Warm-up Free Policy Optimization:\\Improved Regret in Linear Markov Decision Processes}

\author{%
  Asaf Cassel\\
  Tel Aviv University\\
  \texttt{acassel@mail.tau.ac.il}\\
  \And
  Aviv Rosenberg\\
  Google Research\\
  \texttt{avivros@google.com}\\
}

\begin{document}

\maketitle

\begin{abstract}
    % Policy Optimization (PO) methods are among the most popular Reinforcement Learning (RL) algorithms in practice.
    % Recent theoretical works developed PO-based algorithms with rate-optimal regret guarantees under the linear Markov Decision Process (MDP) model.
    % However, these algorithms rely on a costly warm-up reward-free exploration phase that is hard to implement in practice.
    % This paper eliminates this undesired warm-up phase, proposing a simple and efficient truncation mechanism instead.
    % Our PO algorithm achieves rate-optimal regret with improved dependence on the other parameters of the problem (horizon and function approximation dimension) in two fundamental settings: adversarial losses with full-information feedback and stochastic losses with bandit feedback.
    Policy Optimization (PO) methods are among the most popular Reinforcement Learning (RL) algorithms in practice.
    Recently, \citet{sherman2023rate} proposed a PO-based algorithm with rate-optimal regret guarantees under the linear Markov Decision Process (MDP) model.
    However, their algorithm relies on a costly pure exploration warm-up phase that is hard to implement in practice.
    This paper eliminates this undesired warm-up phase, replacing it with a simple and efficient contraction mechanism.
    Our PO algorithm achieves rate-optimal regret with improved dependence on the other parameters of the problem (horizon and function approximation dimension) in two fundamental settings: adversarial losses with full-information feedback and stochastic losses with bandit feedback.
    % Our PO algorithm achieves $
    % \tilde{O}(\sqrt{K (d^3 H^4 + d H^5)})$ regret, improving the best known dependence on the horizon $H$ and function approximation dimension $d$ in two fundamental settings: adversarial losses with full-information feedback and stochastic losses with bandit feedback.
\end{abstract}

\section{Introduction}

Policy Optimization (PO) is a widely used method in Reinforcement Learning (RL) that achieved tremendous  empirical success, with applications ranging from robotics and computer games \citep{schulman2015trust,schulman2017proximal,mnih2015human,haarnoja2018soft} to Large Language Models (LLMs; \cite{stiennon2020learning,ouyang2022training}).
Theoretical work on policy optimization algorithms initially considered tabular Markov Decision Processes (MDPs; \cite{even2009online,neu2010online,shani2020optimistic,luo2021policy}), where the number of states is assumed to be finite and small.
In recent years the theory was generalized to infinite state spaces under function approximation, specifically under linear function approximation in the linear MDP model \citep{luo2021policy,dai2023refined,sherman2023improved,sherman2023rate,liu2023towards}.

Recently, \cite{sherman2023rate} presented the first policy optimization algorithm that achieves rate-optimal regret in linear MDPs, i.e., a regret bound of $\smash{\widetilde{O} (\mathrm{poly}(H,d) \sqrt{K})}$, where $K$ is the number of interaction episodes, $H$ is the horizon, and $d$ is the dimension of the linear function approximation.
However, their algorithm requires a pure exploration warm-up phase to obtain an initial estimate of the transition dynamics.
To that end, they utilize the algorithm of \cite{wagenmaker2022reward} for reward-free exploration which is not based on the policy optimization paradigm.
Moreover, although this algorithm is computationally efficient, it relies on intricate estimation techniques that are hard to implement in practice and unlikely to generalize beyond linear function approximation (see discussion in \cref{sec:algo}).

In this paper, we propose a novel contraction mechanism to avoid this costly warm-up phase.
Both our contraction mechanism and the warm-up phase serve a similar purpose -- ensuring that the Q-value estimates are bounded and yield ``simple'' policies.
But, unlike the warm-up, our method is integrated directly into the PO algorithm, implemented using a simple conditional truncation of the Q-estimates, and only contributes a lower-order term to the final regret bound.
Moreover, our approach is much more efficient in practice since it does not rely on any reward-free methods, which explore the state space uniformly without taking the reward into account.

Based on this contraction mechanism, we build a new policy optimization algorithm that is simpler, more computationally efficient, easier to implement, and most importantly, improves upon the best-known regret bounds for policy optimization in linear MDPs.
Our regret bound holds in two fundamental settings: 
\begin{enumerate}
    \item Adversarial losses with full-information feedback, where the loss function changes arbitrarily between episodes and is revealed to the agent entirely at the end of each episode.

    \item Stochastic losses with bandit feedback, where the loss function in each episode is sampled i.i.d from some unknown fixed distribution and the agent only observes instantaneous losses in the state-action pairs that she visits.
\end{enumerate}
In these settings, the best-known regret bound (by \cite{sherman2023improved}) was $\widetilde{O} (\sqrt{H^7 d^4 K})$.
Our algorithm, Contracted Features Policy Optimization (CFPO), achieves $\widetilde{O} (\sqrt{H^4 d^3 K})$ regret, yielding a $\smash{\sqrt{H^3 d}}$ improvement over any algorithm for the adversarial setting and matching the value iteration based approach of \cite{jin2020provably} in the stochastic setting.
We note that this is the best regret we can hope for without more sophisticated variance reduction techniques \citep{azar2017minimax,zanette2019tighter,he2023nearly,zhang2024horizon}, that have not yet been applied to PO algorithms even in the tabular setting.\footnote{\cite{wu2022nearly} apply variance reduction techniques to get better regret bounds in the tabular setting, but they use $L_2$-regularization instead of KL-regularization which does not align with practical PO algorithms \cite{schulman2015trust,schulman2017proximal}.}
Ignoring logarithmic factors, the regret of CFPO leaves a gap of only $\sqrt{H d}$ from the $\Omega( \sqrt{H^3 d^2 K})$ lower bound for linear MDPs \citep{zhou2021nearly}.
Finally, our analysis relies on a novel regret decomposition that uses a notion of contracted (sub) MDP and may be of separate interest (see \cref{sec:analysis}).

\subsection{Related work}

\paragraph{Policy optimization in tabular MDPs.}
The regret analysis of PO methods in tabular MDPs was introduced by \cite{even2009online}, which considered the case of known transitions and adversarial losses under full-information feedback.
\cite{neu2010on,neu2010online} extended their algorithms to adversarial losses under bandit feedback.
Then, \cite{shani2020optimistic} presented the first PO algorithms for the case of unknown transitions (for both stochastic and adversarial losses), and finally \cite{luo2021policy} devised a PO algorithm with rate-optimal regret for the challenging case of unknown transitions with adversarial losses under bandit feedback.
Since then, PO was studied in more challenging cases, e.g., delayed feedback \citep{lancewicki2020learning,lancewicki2023learning} and best-of-both-worlds \citep{dann2023best}.

\paragraph{Other regret minimization methods in tabular MDPs.}
An alternative popular method for regret minimization in tabular MDPs with adversarial losses is O-REPS \citep{zimin2013online,rosenberg2019bandit,rosenberg2019online,jin2020learning}, which optimizes over the global state-action occupancy measures instead of locally over the policies in each state.
However, this method is hard to implement in practice and does not generalize to the function approximation setting (without restrictive assumptions).
For stochastic losses, optimistic methods based on Value Iteration (VI; \cite{jaksch2010near,azar2017minimax,zanette2019tighter}) and Q-learning \citep{jin2018q,zhang2020almost} are known to guarantee optimal regret, which has not been established yet for adversarial losses.

\paragraph{Policy optimization in linear MDPs.}
While \cite{sherman2023rate} established rate-optimal regret for PO methods in linear MDPs with stochastic losses, most of the recent research focused on the case of adversarial losses with bandit feedback \citep{luo2021policy,neu2021online,dai2023refined,sherman2023improved,kong2023improved,liu2023towards,zhong2023theoretical}, where rate-optimality has not been achieved yet.

\paragraph{Other regret minimization methods in linear MDPs and other models for function approximation.}
Unlike O-REPS methods that do not generalize to linear function approximation, value-based methods (operating under the stochastic loss assumption) are also popular in linear MDPs and have been shown to yield optimal regret \citep{jin2020learning,zanette2020frequentist,wagenmaker2022first,hu2022nearly,he2023nearly,agarwal2023vo}.
Another line of works \citep{ayoub2020model,modi2020sample,cai2020provably,zhang2021improved,zhou2021nearly,zhou2021provably,he2022near,zhou2022computationally} study linear mixture MDP which is a different model that is incomparable with linear MDP \citep{zhou2021provably}.
Finally, there is a rich line of works studying statistical properties of RL with more general function approximation \citep{munos2005error,jiang2017contextual,dong2020root,jin2021bellman,du2021bilinear}, but these usually do not admit computationally efficient algorithms.

\section{Problem setup}
\label{sec:problem-setup}

\paragraph{Episodic Markov Decision Process (MDP).}
A finite-horizon episodic MDP $\M$ is defined by a tuple $(\X, \A, x_1, \{ \ell^k \}_{k=1}^K, P, H)$ with $\X$, a set of states, $\A$, a set of actions, $H$, decision horizon, $x_1 \in \X$, an initial state (assumed to be fixed for simplicity), $\smash{P = (P_h)_{h \in [H]}, P_h: \X \times \A \to \Delta(\X)}$, the transition probabilities, and $\smash{\{ \ell^k \}_{k=1}^K}$, sequence of loss functions such that $\smash{\ell^k = (\ell^k_h)_{h \in [H]}, \ell^k_h: \X \times \A \to [0,1]}$, is a horizon dependent immediate loss function for taking action $a$ at state $x$ and horizon $h$ of episode $k$. 
A single episode $k$ of an MDP is a sequence $\smash{(x^k_h, a^k_h, \ell^k_h(x^k_h,a^k_h))_{h \in [H]} \in (\X \times \A \times [0,1])^H}$ such that
\begin{align*}
    \Pr[x^k_{h+1} = x' \mid x^k_h = x, a^k_h = a] = P_h(x' \mid x,a)
    .
\end{align*}
For the losses, we consider two settings: stochastic and adversarial.
In the stochastic setting, there exists a fixed loss function  $\ell = (\ell_h)_{h \in [H]}, \ell_h: \X \times \A \to [0,1]$ such that $\ell^k$ is sampled i.i.d from a distribution whose expected value is defined by $\ell$, i.e., $\EE\brk[s]*{\ell^k_h(x,a) \mid x, a} = \ell_h(x,a)$. 
In the adversarial setting, the loss function sequence $\{ \ell^k \}_{k=1}^K$ is chosen by an
adaptive adversary.
% Notice the overloaded notation for $\ell^k_h$ where $\ell^k_h(\cdot, \cdot)$ refers to the immediate loss function, and $\ell^k_h$ (without inputs) refers to a sampled loss at horizon $h$.

\paragraph{Linear MDP.}
A linear MDP \cite{jin2020provably} satisfies all the properties of the above MDP but has the following additional structural assumptions. 
There is a known feature mapping $\phi : \mathcal{X} \times A \to \RR[d]$ such that $P_h(x' \mid x,a) = \phi(x,a)\tran \psiH(x')$ where $\psiH : \mathcal{X} \to \RR[d]$ are unknown parameters.
Moreover, for all $h \in [H], k \in [K]$, there is an unknown vector $\thetaKH \in \RR[d]$ such that, in the adversarial case, $\ell_h^k(x,a) = \phi(x,a)\tran \thetaKH$, while in the stochastic case, $\thetaKH = \thetaH$ and $\ell_h(x,a) = \phi(x,a)\tran \thetaH$.
% \begin{align*}
%     \ell_h^k(x,a) = \phi(x,a)\tran \thetaKH
%     ,
%     \quad
%     P_h(x' \mid x,a) = \phi(x,a)\tran \psiH(x')
%     ,
% \end{align*}
% where $\thetaKH \in \RR[d]$ and $\psiH : \mathcal{X} \to \RR[d]$ are unknown parameters.
%
We make the following normalization assumptions, common throughout the literature:
\begin{enumerate}
    \item $\norm{\phi(x,a)} \le 1$ for all $x \in X, a \in \A$;
    \item $\norm{\thetaKH} \le \sqrt{d}$ for all $h \in [H], k \in [K]$;
    \item $\norm{\abs{\psiH}(\X)} = \norm{\sum_{x \in \X} \abs{\psiH(x)}} \le \sqrt{d}$ for all $h \in [H]$;
\end{enumerate}
where $\abs{\psiH(x)}$ is the entry-wise absolute value of $\psiH(x) \in \RR[d]$.
We follow the standard assumption in the literature that the action space $\A$ is finite.
In addition, without loss of generality (see \cite{cassel2024near} for details), we also assume that the state space $\X$ is finite.
% In addition, for ease of mathematical exposition, we also assume that the state space $\X$ is finite.
% This allows for simple matrix notation and avoids technical measure theoretic definitions.
% Importantly, our results are completely independent of the state space size $\abs{\X}$, both computationally and in terms of regret.
% Thus, there is no particular loss of generality. 

\paragraph{Policy and value.} A stochastic Markov policy $\pi = (\pi_h)_{h \in [H]} : [H] \times \X \mapsto \Delta(\A)$ is a mapping from a step and a state to a distribution over actions. Such a policy induces a distribution over trajectories $\iota = (x_h,a_h)_{h \in [H]}$, i.e., sequences of $H$ state-action pairs. For $f : (\X \times \A)^H \to \RR$, which maps trajectories to real values, we denote the expectation with respect to $\iota$ under dynamics $P$ and policy $\pi$ as $\EE[P, \pi] \brk[s]{f(\iota)}$. Similarly, we denote the probability under this distribution by $\PR[P, \pi] \brk[s]{\cdot}$. We denote the class of stochastic Markov policies as $\piClass$.
For any policy $\pi \in \piClass$, horizon $h \in [H]$ and episode $k \in [K]$ we define its loss-to-go, as
\begin{align*}
V_h^{k,\pi}(x)
=
\EE[P,\pi]\brk[s]*{
    \sum_{h'=h}^{H} 
    \EE\brk[s]{\ell_{h'}^k(x_{h'},a_{h'}) \mid x_{h'}, a_{h'}}
    \;\bigg|\; x_h = x
}
,
\end{align*}
which is the expected loss if one starts from state $x \in \X$ at horizon $h$ of episode $k$ and follows policy $\pi$ onwards. Note that the inner expectation is only relevant for stochastic losses as its argument is deterministic in the adversarial setup.
The performance of a policy in episode $k$, also known as its value, is measured by its expected cumulative loss $V_1^{k,\pi}(x_1)$.
% When clear from context, we omit the $1$ and simply write $V^{k,\pi}(x_1)$. 
% Finally, the value function associated with the loss of episode $k \in [K]$ is denoted by $V^{k,\pi}_h(x) = V_h^{\pi}(x ; \ell^k)$.

\paragraph{Interaction protocol and regret.}
We consider a standard episodic regret minimization setting where an algorithm performs $K$ interactions with an MDP $\M$.
For stochastic losses we consider bandit feedback, where the agent observes only the instantaneous losses along its trajectory, while for adversarial losses we consider full-information feedback, where the agent observes the full loss function $\ell^k$ in the end of episode $k \in [K]$.
Concretely, at the start of each interaction/episode $k \in [K]$, the agent specifies a stochastic Markov policy $\piK = (\piK[h])_{h \in [H]}$. Subsequently, it observes the trajectory $\iota^k$ sampled from the distribution $\PR[P, \piK]$, and, either the individual episode losses $\ell_h^k(x^k_h,a^k_h), h \in [H]$ in the case of bandit feedback, or the entire loss function $\ell^k$ in the case of full-information feedback.

We measure the quality of any algorithm via its \emph{regret} -- the difference between the value of the policies $\piK$ generated by the algorithm and that of the best policy in hindsight, i.e.,
\begin{align*}
    \regret
    =
    \sum_{k=1}^{K} V_1^{k,\piK}(x_1) - \min_{\pi \in \piClass} \sum_{k=1}^{K} V_1^{k,\pi}(x_1)
    =
    \sum_{k=1}^{K} V_1^{k,\piK}(x_1) - V_1^{k,\piOpt}(x_1)
    ,
\end{align*}
where the best policy in hindsight is denoted by $\piOpt$ (known to be optimal even among the class of stochastic history-dependent policies). 
% Finally, we use the notation $V^\star = V_1^{\piOpt}$ for the optimal value.

\paragraph{Notation.} 
Throughout the paper
$
    \phi_h^k
    =
    \phi(x_h^k, a_h^k) \in \RR[d]
$
% and
% $
%     \ell_h^k
%     =
%     \ell_h^k(x_h^k, a_h^k)
%     \in [0,1]
% $
denote the state-action features 
% and losses
at horizon $h$ of episode $k$. 
In addition, $\norm{v}_A = \sqrt{v\tran A v}$.
Hyper-parameters follow the notations $\beta_z$ and $\eta_z$ for some $z$, and $\delta \in (0,1)$ denotes a confidence parameter. Finally, in the context of an algorithm, $\gets$ signs refer to compute operations whereas $=$ signs define operators, which are evaluated at specific points as part of compute operations.

\section{The role of value clipping}
\label{sec:role-of-clipping}

Before presenting our contraction technique and main results, we discuss the role that value clipping plays in regret minimization and its apparent necessity for linear MDPs.
% , and the lack thereof in the tabular setting.
%
As a starting point, it is important to note that, while commonly used \citep{azar2017minimax,luo2021policy}, value clipping is not strictly necessary in tabular MDPs. To demonstrate this, consider a fairly standard optimistic Value Iteration (VI) algorithm that
% estimates the dynamics and losses using the sample means $\hat{P},\hat{\ell}$.
% 
% 
% At a conceptual level, optimism-based approaches work by optimizing a sequence of optimistic MDPs $\bar{\mathcal{M}}^k = ((\X, \A, x_1, \bar{\ell}^k, \bar{P}^k, H)$.  that are optimistic in the sense that $V^{\piOpt}(\bar{\mathcal{M}}^k) - V^{\piOpt}(\mathcal{M}) \le 0$, but not overly optimistic in the sense that $V^{\piK}(\mathcal{M}) - V^{\piK}(\bar{\mathcal{M}}^k) \le \beta^k$.
% 
% 
% 
% Direct or indirect value clipping is prevalent in many regret minimization algorithms for MDPs \aviv{cite}. 
% 
% In the tabular setting, its purpose is to decouple the bound on the value uncertainty from the magnitude of the added exploration bonuses. 
% To see that, suppose an algorithm 
constructs sample-based estimates $\hat{\ell}, \hat{P}$ with error bounds $\Delta \ell, \Delta P$, defines exploration bonuses $b = (\Delta \ell + H \cdot\Delta P)$, and chooses a policy $\hat{\pi}^\star$ that is optimal in the empirical MDP whose dynamics are $\hat{P}$ and losses are $\hat{\ell}-b$.
Then its single-episode regret may be decomposed as
\begin{align*}
    V_1^{\hat{\pi}^\star}(x_1) - V_1^{\piOpt}(x_1)
    =
    \underbrace{
    V_1^{\hat{\pi}^\star}(x_1) - \hat{V}_1^{\hat{\pi}^\star}(x_1)
    }_{(i) - \text{bias / cost of optimism}}
    +
    \underbrace{
    \hat{V}_1^{\hat{\pi}^\star}(x_1) - \hat{V}_1^{\piOpt}(x_1)
    }_{(ii) - \text{FTL / ERM}}
    +
    \underbrace{
    \hat{V}_1^{\piOpt}(x_1) - V_1^{\piOpt}(x_1)
    }_{(iii) - \text{optimism}}
    ,
\end{align*}
where $\hat{V}$ is the value under the empirical MDP.
Now, by definition of $\hat{\pi}^\star$, we have that $(ii) \le 0$. Next, using a standard value difference lemma (\cref{lemma:extended-value-difference} in \cref{sec:technical}) we have that $(i) \lesssim b$ and
\begin{align}
    \label{eq:clipping-free-value-difference}
    (iii)
    % \\
    &
    =
    \EE[\hat{P}, \piOpt] \brk[s]*{
    \sum_{h \in [H]} \Delta {\ell}(x_h,a_h) - b(x_h,a_h) + \sum_{x' \in \X} \Delta {P}(x' \mid x_h, a_h) {V}^{\piOpt}_{h+1}(x')
    }
    \\
    \nonumber
    &
    \le
    \EE[\hat{P}, \piOpt] \brk[s]*{
    \sum_{h \in [H]} \Delta {\ell}(x_h,a_h) + H \Delta P(x_h, a_h)
    -
    b(x_h,a_h)
    }
    =
    0
    ,
\end{align}
where the inequality also used that $V^{\piOpt}_h \in [0, H]$.
The final regret bound is concluded by summing over $k \in [K]$ and using a bound on harmonic sums. We note that a similar clipping-free method also works for tabular PO (see \cite{cassel2024near}).

Moving on to Linear MDPs, one might expect a similar approach to work. Unfortunately, the standard approach that estimates the dynamics backup operators $\psiH, h \in [H]$ using regularized least-squares presents a significant challenge. This is because, unlike the tabular setting, the resulting estimate $\hat{P}_h(\cdot \mid x,a)= \phi(x,a)\tran\hatPsiH(\cdot)$ (\cref{eq:LS-estimate-psi}) is not guaranteed to yield a valid probability distribution, i.e., there could exist $x \in \X, a \in \A, h \in [H]$ such that
\begin{align*}
    % \sum_{x' \in \X} 
    \norm{\hat{P}_h(\cdot \mid x,a)}_1 = c > 1
    \quad
    \text{and/or}
    \quad
    \min_{x' \in \X} \hat{P}_h(x' \mid x,a) < 0
    .
\end{align*}
$\hat{P}$ is still a finite signed-measure, which is enough for the first equality in \cref{eq:clipping-free-value-difference} to hold. However, since $\smash{\EE[\hat{P},\piOpt]}$ could contain negative probability terms, the inequality in \cref{eq:clipping-free-value-difference} does not hold. These negative probabilities also seem to make calculating $\hat{\pi}^\star$ computationally hard. Finally, the $\ell_1-$norm exceeding $1$ may cause term $(i)$ to depend on $H$ exponentially. While some of these issues could be mitigated without clipping, we are not aware of a method that resolves all simultaneously.

The use of value clipping opens the path for an alternative value decomposition that replaces $\EE[\hat{P}, \piOpt]$ in \cref{eq:clipping-free-value-difference} with $\EE[P, \piOpt]$ at the cost of also replacing $\smash{V_{h+1}^{\piOpt}}$ with $\smash{\hat{V}_{h+1}^{\piOpt}}$. We thus need that $\smash{\abs{\hat{V}_{h+1}^{\piOpt}} \lesssim H}$ for the inequality in \cref{eq:clipping-free-value-difference} to work. This is made possible using a clipping mechanism that decouples the scale of $\smash{\hat{V}_{h+1}^{\piOpt}}$ from the magnitude of the bonuses $b$, which may be much larger when the error estimates $\Delta \ell, \Delta P$ are large.
This is typically achieved by adding $\max\brk[c]{0, \cdot}$ to the recursive formula for the value function. 
A similar clipping approach also works for tabular PO and VI \citep{azar2017minimax,luo2021policy}, and even for VI in linear MDPs \citep{jin2020provably}.

However, this is not the case for PO in linear MDPs where \cite{sherman2023rate} explain that this type of value clipping leads to prohibitive complexity of the policy and value function classes, and thus sub-optimal regret. They overcome this issue using a warm-up based truncation technique. In what follows, we suggest an alternative solution that uses a novel notion of contracted features and has several advantages over their approach
(see discussion at the end of \cref{sec:algo}).

\section{Algorithm and main result}
\label{sec:algo}

We present Contracted Features Policy Optimization (CFPO; \cref{alg:r-opo-for-linear-mdp-regular-bonus}), a policy optimization routine for regret minimization in linear MDPs.
The algorithm operates in epochs, each beginning when the uncertainty of the dynamics estimation shrinks by a multiplicative factor, as expressed by the determinant of the covariance matrices $\covKH, h \in [H]$ (see \cref{eq:covKH} for the definition of $\covKH$ and \cref{line:repo-epoch-condition} for the epoch change condition).
At the start of each epoch $e$, we reset the policy to its initial (uniform) state, and define the contracted features $\smash{\phiBarH^{\kEpoch}, h \in [H]}$ (\cref{eq:contracted-features}) by multiplying the original features with coefficients in the range $[0,1]$, and thus shrinking their distance to the origin. Inspired by ideas from \cite{zanette2020frequentist}, these coefficients are chosen inversely proportional to the current uncertainty of the least squares estimators in each state-action pair, essentially degenerating the MDP in areas of high uncertainty.
% and are iteratively refined between epochs.
% \aviv{I think we should explain in words here what contracted features mean}
%
Inside an epoch, at episode $k$, we compute the estimated reward vector $\smash{\hatThetaK}$ (\cref{eq:LS-estimate-theta}) and estimated dynamics backup operators $\smash{\psiKH}$ (\cref{eq:LS-estimate-psi}).
Then, we use these $\smash{\hatThetaK}$ and $\smash{\psiKH}$ to compute our Q-value estimates with the contracted features (\cref{eq:value-iteration-Q}), and run an online mirror descent (OMD) update over them (\cref{eq:policy-omd-update}), i.e., run a policy optimization step with respect to the contracted empirical MDP (more on this in \cref{sec:contracted-mdp}).

\begin{algorithm}[t]
	\caption{Contracted Features PO for linear MDPs} %and reward-free warm-up}
 \label{alg:r-opo-for-linear-mdp-regular-bonus}
	\begin{algorithmic}[1]
	    \STATE \textbf{input}:
		$d, H, K, \A, \delta, \betaWarmup, \betaB, \etaO > 0$.

            \STATE \textbf{initialize}: $e \gets -1, \covKH[1] \gets I, h \in [H]$.
        
	    \FOR{episode $k = 1,2,\ldots, K$}
     
                \IF{\mbox{$k = 1$
                    \textbf{ or }
                    $\exists h \in [H], \  \det\brk{\covKH} \ge 2 \det\brk{\hatCovKH}$}
                }
                \label{line:repo-epoch-condition}

                    \STATE $e \gets e + 1$ and $\kEpoch \gets k$.

                    \STATE
                    \label{eq:contracted-features}
                    $\bar{\phi}_h^{\kEpoch}(x,a)
                    =
                    \phi(x,a) \cdot \sig\brk*{-\betaWarmup\norm{\phi(x,a)}_{(\hatCovKH)^{-1}} + \log K}$.
                    \hfill\COMMENT{$\sig(z) = 1 / (1 + \exp(-z))$}

                    \STATE $\pi^k_h(a \mid x) = 1 / \abs{\A}$ for all $h \in [H], a \in \A, x \in \X.$
                \ENDIF

            \STATE Play $\piK$ and
            observe losses $(\ell_h^k(x^k_h,a^k_h))_{h \in [H]}$ and trajectory $\iota^k = (x^k_h,a^k_h)_{h \in [H]}$.

            \STATE In the case of full-information feedback: observe $\thetaKH$.
            
            \STATE Define $\hat{V}_{H+1}^{k}(x) = 0$ for all $x \in \mathcal{X}$.
            
            \FOR{$h = H, \ldots, 1$}

                \STATE \label{eq:covKH}$\covKH[k+1] \gets I + \sum_{\tau \in [k]} \phi_h^\tau (\phi_h^\tau)\tran.$ 

                \STATE \label{eq:LS-estimate-theta}$\hatThetaKH \gets
                \begin{cases}
                    (\covKH)^{-1} \sum_{\tau \in [k-1]} \phi_h^\tau \ell_h^\tau(x^\tau_h,a^\tau_h)
                    ,
                    &
                    \mathrm{feedback}=bandit
                    \\
                     \thetaKH
                     ,
                    &
                    \mathrm{feedback}=full.
                \end{cases}
                $
                
            \STATE For any $V: \X \to \RR, x \in \X, a \in \A$ define:
            \begin{alignat}{2}
                    % &
                    % \qquad
                    % w_h^{k}
                    % &&
                    % \gets
                    % \hatThetaKH
                    % +
                    % \psiKH \hat{V}^{k}_{h+1}
                    % \\ 
                    \label{eq:LS-estimate-psi}
                    &
                    \qquad
                    \psiKH V 
                    &&
                    =
                    (\covKH)^{-1} \sum_{\tau \in [k-1]} \phi_h^\tau V(x_{h+1}^\tau),
                    \\
                    \label{eq:value-iteration-Q}
                    &
                    \qquad
                    \hat{Q}_h^{k}(x,a)
                    &&
                    =
                    \bar{\phi}_h^{\kEpoch}(x,a)\tran \brk[s]{
                    \hatThetaKH
                    +
                    \psiKH \hat{V}^{k}_{h+1}
                    }
                    -
                    \betaB \norm{\bar{\phi}_h^{\kEpoch}(x,a)}_{(\hatCovKH)^{-1}}
                    ,
                    \\
                    \label{eq:value-iteration-V}
                    &
                    \qquad
                    \hat{V}_h^{k}(x)
                    &&
                    =
                    \sum_{a \in \mathcal{A}} \pi^{k}_h(a \mid x) \hat{Q}_h^{k}(x,a)
                    ,
                    \\
                    \label{eq:policy-omd-update}
                    &
                    \qquad
                    \pi^{k+1}_h(a \mid x)
                    &&
                    \propto
                    \pi^{k}_h(a \mid x) \exp (-\etaO \hat{Q}_h^{k}(x,a))
                    .
                \end{alignat}
            \ENDFOR       
	       
	    \ENDFOR
	\end{algorithmic}
	
\end{algorithm}
We note that the computational complexity of \cref{alg:r-opo-for-linear-mdp-regular-bonus} is comparable to other algorithms for regret minimization in linear MDPs, such as LSVI-UCB \citep{jin2020provably}.
The following is our main result for \cref{alg:r-opo-for-linear-mdp-regular-bonus} (see the full analysis \cref{appendix-sec:analysis}).
\begin{theorem}
\label{thm:regret}
    Suppose that we run CFPO (\cref{alg:r-opo-for-linear-mdp-regular-bonus}) with the parameters defined in \cref{thm:regret-bound-PO-linear-regular-bonus} (in \cref{appendix-sec:analysis}).
    Then, with probability at least $1 - \delta$, we have
    \begin{align*}
        \regret
        =
        O \left(
        \sqrt{H^4 d^3 K \log(K) \log (K H / \delta)}
        +
        \sqrt{H^5 d K \log (K) \log \abs{\A}}
        \right)
        .
    \end{align*}
\end{theorem}

\paragraph{Discussion.}
Policy optimization algorithms typically entail running OMD over estimates $\hat{Q}$ of the state-action value function $Q$, as in \cref{eq:policy-omd-update}. The crux of the algorithm is in obtaining such estimates that satisfy an optimistic condition similar to \cref{eq:clipping-free-value-difference}, while also keeping the complexity of the policy class bounded. As discussed in \cite{sherman2023rate}, the latter depends on $\sum_{k' \in [k]}\hat{Q}_h^{k'}$ (\cref{eq:value-iteration-Q}) having a low dimensional representation nearly independent of $k$. Although standard unclipped estimates admit such a representation, they lack other essential properties (see discussion in \cref{sec:role-of-clipping}).
On the other hand, the standard clipping method, which restricts the value to $[0,H]$ between each backup operation (see, e.g., \cite{jin2020provably}), does not admit the desired representation.

\cite{sherman2023rate} overcame this issue by employing a warm-up phase based on a reward-free pure exploration algorithm by \cite{wagenmaker2022reward} to obtain initial backup operators $\smash{\widehat{\psi}^0_h, h \in [H]}$ and subsets $\bar{\X}_h \subseteq \X, h \in [H]$ such that: (i) for every $x,a \in \bar{\X}_h \times \A$ the bonuses (b in \cref{sec:role-of-clipping}), which are proportional to the estimation uncertainty of the value backup estimates, are small ($\le 1$); and (ii) for all policies $\pi \in \piClass$, the probability of reaching any $x,a \notin  \cup_{h \in [H]}\bar{\X}_h \times \A$ is small ($\lesssim K^{-1/2}$). To ensure that the overall value estimates remain bounded, they truncate (zero out) the Q-value estimate of these nearly unreachable state-action pairs, an operation that allows for a low-dimensional representation of the policies.
% Together, these imply that their clipping works for all policies, a typically crucial property for optimism claims, which rely on the uncertainty with respect to $\piOpt$ rather than the chosen policy $\piK$.
%
Nonetheless, their warm-up approach has several drawbacks. 
\begin{itemize}
    \item It runs for $K_0 = \mathrm{poly}(d, H)\sqrt{K}$ episodes, contributing the leading term in their regret guarantee;

    \item It relies on a first-order regret algorithm by \cite{wagenmaker2022first} that is not PO-based and uses a computationally hard variance-aware Catoni estimator for robust mean estimation of the value backups, instead of the standard least-squares estimator. To maintain computational efficiency, they use an approximate version of the estimator, losing a factor of $\smash{\sqrt{d}}$ in the regret;
    
    \item Still, to the best of our knowledge, even the approximate estimator must be computed using binary search methods, making it hard to apply in practical methods that typically rely on gradient-based continuous optimization techniques;

    \item It runs separate algorithms for each horizon $h \in [H]$, using only $1$ out of $H$ samples during the warm-up phase;
    
    \item It is not reward-aware, and thus has to explore the space uniformly to ensure that the uncertainty is small for all policies, which could be highly prohibitive in practice.
    % \item \asaf{it does not use PO in the warm-up.}
\end{itemize}
Our feature contraction approach obtains the desired bounded Q-value estimates and low-complexity policy class without relying on a dedicated warm-up phase.
% (see \cref{sec:analysis}).
Crucially, it only contributes a lower order term of $\mathrm{poly}(d, H) \log K$ to the regret guarantee, thus improving the overall dependence on $d$ and $H$. Additionally, it uses all samples, is easy to implement, and is reward-aware. To understand the benefit of reward-awareness, consider an MDP where at the initial state the agent has two actions, each leading to a distinct MDP. Now, suppose that both MDPs have only a single state and action for the first $H/2$ steps with one MDP incurring a loss of $1$ in these steps while the other incurring $0$ loss. Notice that regardless of the last $H/2$ steps, the $0$ loss MDP will outperform the $1$ loss MDP. Nonetheless, the reward-free warm-up, which does not observe the losses, will have to fully explore both MDPs. In contrast, our reward-aware approach would quickly stop exploring the inferior MDP, leading to better performance in practice.

\section{Analysis}\label{sec:analysis}
In this section, we prove the main claims of our result. For full details see \cref{appendix-sec:analysis}. We begin by introducing the main technical tool for our contraction mechanism -- the contracted MDP.

\subsection{Contracted (sub) MDP}\label{sec:contracted-mdp}
% We define the following (sub) MDP. For any $x \in \X, a \in \A, h \in [H]$ let $\rho_h(x,a) \in [0,1]$ be arbitrary. For any MDP with losses $\ell_h : \X \times \A \to [0,1]$ and transition probability $P_h : \X \times \A \to \Delta(\X)$, consider the (sub) MDP whose losses are defined by $\bar{\ell}_h(x,a) = \rho_h(x,a)\ell_h(x,a) \in [0,1]$ and has (sub) probability transitions defined as $\bar{P}_h(x' \mid x,a) = \rho_h(x,a) P_h(x' \mid x,a) \in [0,1].$ 
For any MDP $\M = (\X, \A, x_1, \{ \ell^k \}_{k=1}^K, P, H)$ and contraction coefficients $\rho: [H] \times \X \times \A \to [0,1]$ we define a contracted (sub) MDP $\bar{\M}(\rho) = (\X, \A, x_1, \{ \bar{\ell}^k \}_{k=1}^K, \bar{P}, H)$ where as $\bar{\ell}_h^k(x,a) = \rho_h(x,a)\ell_h^k(x,a) \in [0,1]$ are the contracted losses and $\bar{P}_h(x' \mid x,a) = \rho_h(x,a) P_h(x' \mid x,a) \in [0,1]$ are the contracted (sub) probability transitions.
Notice that the transitions being a sub-probability measure implies that $\sum_{x' \in \X} P_h(x' \mid x,a) \le 1$ as compared with a probability measure where this holds with equality. 
For any Markov policy $\pi \in \piClass$, let $\smash{\bar{V}_h^{k,\pi}(\cdot ; \rho) : \X \to \RR, h \in [H]}$ be the loss-to-go (or value) functions of the contracted MDP. In particular, these may be defined by the usual backward recursion
\begin{align*}
    \bar{V}_h^{k,\pi}(x ; \rho)
    =
    \EE[a \sim \pi(\cdot \mid x)]\brk[s]*{
    \EE\brk[s]{\bar{\ell}_h^k(x,a) \mid x,a }
    +
    \sum_{x' \in \X} \bar{P}_h(x' \mid x,a) \bar{V}_{h+1}^{k,\pi}(x' ; \rho)
    }
    ,
\end{align*}
with $\bar{V}_{H+1}^{k,\pi}(x ; \rho) = 0$ for all $x \in \X$.
%
% We note that for linear MDPs, defining 
%
The following result shows that the value of any contracted MDP lower bounds its non-contracted variant.
\begin{lemma}
\label{lemma:contracted-mdp-is-optimistic}
    For any $\rho: [H] \times \X \times \A \to [0,1], \pi \in \piClass, h \in [H], k \in [K]$, and $x \in \X$ we have that
    $
    \bar{V}_h^{k,\pi}(x ; \rho) \le V_h^{k,\pi}(x)
    .
    $
\end{lemma}
\begin{proof}
    The proof follows by backward induction on $h \in [H+1]$. For the base case $h = H+1$, both values are $0$ and the claim holds trivially. Now suppose the claim holds for $h+1$, then we have that for all $x \in \X$
    \begin{align*}
        \bar{V}_h^{k,\pi}(x ; \rho)
        &
        =
        \EE[a \sim \pi(\cdot \mid x)]\brk[s]*{
        \EE\brk[s]{\bar{\ell}_h^k(x,a) \mid x,a}
        +
        \sum_{x' \in \X} \bar{P}_h(x' \mid x,a) \bar{V}_{h+1}^{k,\pi}(x' ; \rho)
        }
        \\
        &
        \le
        \EE[a \sim \pi(\cdot \mid x)]\brk[s]*{
        \EE\brk[s]{{\ell}_h^k(x,a) \mid x,a}
        +
        \sum_{x' \in \X} {P}_h(x' \mid x,a) {V}_{h+1}^{k,\pi}(x')
        }
        % \\
        % &
        =
        V_h^{k,\pi}(x)
        .
        \qedhere
    \end{align*}
    % thus concluding the induction step and the proof.
\end{proof}
Next, for any epoch $e \in \setEpochs$, consider its contracted linear MDP (\cref{eq:contracted-features} in \cref{alg:r-opo-for-linear-mdp-regular-bonus}) whose contraction coefficients are
$
    \rho_h^{\kEpoch}(x,a)
    =
    \sig\brk*{-\betaWarmup\norm{\phi(x,a)}_{(\hatCovKH)^{-1}} + \log K}
    .
$
The following result gives an upper bound on the performance gap between the contracted and non-contracted variants.
\begin{lemma}
\label{lemma:cost-of-contraction}
    For any $e \in \setEpochs$ and $v \in \RR[d]$ we have that
    \begin{align*}
        (\phi(x_h,a_h) - \phiBarH^{\kEpoch}(x_h,a_h))\tran v
        \le
        \brk{
        4 \betaWarmup^2 \norm{\phi(x_h,a_h)}_{(\covKH)^{-1}}^2
        +
        2 K^{-1}
        } \abs*{\phi(x_h,a_h)\tran v}
        .
    \end{align*}
\end{lemma}
\begin{proof}
    We have that
    \begin{align*}
        (\phi(x_h,a_h) - \phiBarH^{\kEpoch}(x_h,a_h))\tran v
        &
        =
        \sig\brk{\betaWarmup\norm{\phi(x_h,a_h)}_{(\hatCovKH)^{-1}} - \log K}
        \cdot{\phi(x_h,a_h)\tran v}
        \\
        &
        \le
        2\brk{\betaWarmup^2 \norm{\phi(x_h,a_h)}_{(\hatCovKH)^{-1}}^2 
        +
        K^{-1}}\abs*{\phi(x_h,a_h)\tran v}
        \\
        &
        \le
        \brk{
        4 \betaWarmup^2 \norm{\phi(x_h,a_h)}_{(\covKH)^{-1}}^2
        +
        2 K^{-1}
        } \abs*{\phi(x_h,a_h)\tran v}
        ,
    \end{align*}
    where the first relation is by the property of the sigmoid $1 - \sig(x) = \sig(-x)$, the second is by a simple algebric argument that a quadratic function bounds the sigmoid (\cref{lemma:logistic-to-quadratic} in \cref{sec:technical}), and the last relation uses $\det(\covKH) \preceq 2\det(\hatCovKH)$ by \cref{line:repo-epoch-condition} in \cref{alg:r-opo-for-linear-mdp-regular-bonus} (see \cref{lemma:matrix-norm-inequality}  in \cref{sec:technical}).
\end{proof}
We note that the corresponding claim in \cite{sherman2023rate} shows that for all $\pi \in \piClass$
\begin{align}
\label{eq:sherman-cost-of-truncation}
    \EE[P, \pi]\brk[s]{(\phi(x_h,a_h) - \phiBarH^{\kEpoch}(x_h,a_h))\tran v}
    \le
    \epsilon \abs*{\phi(x_h,a_h)\tran v}
    ,
\end{align}
where $\epsilon \approx K^{-1/2}$. Summing this over $k \in [K]$ yields a term that scales as $\sqrt{K}$. In contrast, we use a standard bound on elliptical potentials (\cref{lemma:elliptical-potential}  in \cref{sec:technical}) to get that
\begin{align*}
    \sum_{k \in [K]}
    \brk{
        4 \betaWarmup^2 \norm{\phi(x_h^k,a_h^k)}_{(\covKH)^{-1}}^2
        +
        2 K^{-1}
        }
        \lesssim
        \log K
        .
\end{align*}
This implies that the cost of our contraction is significantly lower than the truncation of \cite{sherman2023rate}. We achieve this reduced cost by using a quadratic (rather than linear) bound on the logistic function.
The challenge in our approach is that the above bound only holds for the observed trajectories rather than for all policies as in \cite{sherman2023rate}. In what follows, we overcome this challenge using a novel regret decomposition.

\subsection{Regret bound} 
\label{sec:regret-bound-main}
For any epoch $e \in [E]$, let $\setEpochsE$ be the set of episodes that it contains, and let $\bar{V}^{k,\pi}_1(x_1; \rho^{\kEpoch})$ denote the value of its contracted MDP as defined above and in \cref{eq:contracted-features} of \cref{alg:r-opo-for-linear-mdp-regular-bonus}.
We bound the regret as
\begin{align*}
    \regret
    & =
    \sum_{k \in [K]} V^{k,\pi^k}_1(x_1) - V^{k,\piOpt}_1(x_1)
    \\
    \tag{\cref{lemma:contracted-mdp-is-optimistic}}
    & \le
    \sum_{e \in \setEpochs}\sum_{k \in \setEpochsE} V^{k,\pi^k}_1(x_1) - \bar{V}^{k,\piOpt}_1(x_1; \rho^{\kEpoch})
    \\
    & 
    =
    \sum_{k \in [K]} V^{k,\pi^k}_1(x_1) - \hat{V}^{k}_1(x_1) 
    +
    \sum_{e \in \setEpochs} \sum_{k \in \setEpochsE} 
    \hat{V}^{k}_1(x_1) - \bar{V}^{k,\piOpt}_1(x_1; \rho^{\kEpoch})
    \\
    &
    =
    \underbrace{
    \sum_{k \in [K]} V^{k,\pi^k}_1(x_1) - \hat{V}^{k}_1(x_1)
    }_{(i) - \text{Bias / Cost of optimism}}
    \\
    &
    \quad
    +
    \underbrace{\sum_{e \in \setEpochs} \sum_{h \in [H]}  \EE[\pBarKe, \piOpt]\brk[s]*{ \sum_{k \in \setEpochsE} \sum_{a \in \A} \hat{Q}^{k}_h(x_h,a) \brk{ \pi^{k}_h(a \mid x_h) - \pi^\star_h(a \mid x_h) } } }_{(ii) - \text{OMD regret}}
    \\
    & \quad +
    \underbrace{\sum_{e \in \setEpochs} \sum_{k \in \setEpochsE} \sum_{h \in [H]} \EE[\pBarKe, \piOpt]\brk[s]*{ 
    \hat{Q}^{k}_h(x_h,a_h)  - \phiBarH^{\kEpoch}(x_h,a_h)\tran \brk{ \thetaKH + \psiH \hat{V}^{k}_{h+1} } } }_{(iii) - \text{Optimism}}
    ,
\end{align*}
where the last relation is by the extended value difference lemma (see \cite{shani2020optimistic} and \cref{lemma:extended-value-difference} in \cref{sec:technical}). 
This decomposition is very similar to the standard one for PO algorithms, but with the crucial difference that term $(iii)$ depends on the contracted features $\phiBarH^{\kEpoch}(x_h,a_h)$ instead of the true features $\phi(x_h,a_h)$.
As a by-product, the expectation in terms $(ii)$ and $(iii)$ is taken with respect to the contracted MDP instead of the true one.
% This adds technical difficulties but does not hurt the final regret bound. 
The purpose of this modification will be made clear in the proof of optimism (see \cref{lemma:optimism-PO-linear-regular-bonus}).

In what follows, we bound each term deterministically, conditioned on the following ``good event'':
\begin{flalign}
    \label{eq:goodRewardPOLinearRegularBonus}
     &
     E_1 = \brk[c]*{\forall k \in [K], h \in [H] : \norm{\thetaKH - \hatThetaKH}_{\covKH} \le \betaR};
     &
     \\
     \label{eq:goodTransitionPOLinearRegularBonus}
     &
     E_2 = \brk[c]*{k \in [K] , h \in [H] : \norm{(\psiH - \psiKH) \hat{V}_{h+1}^{k}}_{\covKH} \le \betaP, \norm{\hat{Q}_{h+1}^{k}}_\infty \le 2H}
     .
     &
\end{flalign}
$E_1$ and $E_2$ are error bounds on the loss and dynamics estimation, respectively. In the full feedback setting, $E_1$ holds trivially with $\betaR=0$. In the bandit setting, it holds with high probability with $\betaR = O(\sqrt{d\log(KH/\delta)})$ by well-established bounds for regularized least-squares estimation \citep{abbasi2011improved}. Showing that $E_2$ holds with high probability follows similarly to \cite{sherman2023rate}, again using least-squares arguments but also using the contraction to ensure that $\hat{Q}_h^k$ are bounded
% and thus $\betaP$ does not grow exponentially with $H$ 
(see \cref{lemma:good-event-PO-linear-regular-bonus} in \cref{appendix-sec:analysis} for details), specifically $\betaP = O(Hd\sqrt{\log(KH/\delta)})$.
The proof of \cref{thm:regret} is concluded by bounding each of the terms in the regret decomposition, summing over $k \in [K]$ and using a standard bound on elliptical potentials (\cref{lemma:elliptical-potential}  in \cref{sec:technical}). Term $(ii)$ is bounded using a standard Online Mirror Descent (OMD) argument (\cref{lemma:omd-term-PO-linear} in \cref{appendix-sec:analysis}).

\paragraph{Optimism and its cost.}
The following lemmas bound terms $(iii)$ and $(i)$, respectively.
\begin{lemma}[Optimism]
\label{lemma:optimism-PO-linear-regular-bonus}
    Suppose that \cref{eq:goodRewardPOLinearRegularBonus,eq:goodTransitionPOLinearRegularBonus} hold, then
    \begin{align*}
        \hat{Q}^{k}_h(x,a)  - \phiBarH^{\kEpoch}(x,a)\tran \brk{ \thetaKH + \psiH \hat{V}^{k}_{h+1} }
        \le
        0
        \quad
        ,
        \forall h \in [H], k \in [K], x \in \X, a \in \A
        .
    \end{align*}
\end{lemma}

\begin{proof}
    We have that
    \begin{align*}
        \hat{Q}^{k}_h(x,a) - \phiBarH^{\kEpoch}(x,a)\tran \brk{ \thetaH + \psiH \hat{V}^{k}_{h+1} } 
        &
        =
        \phiBarH^{\kEpoch}(x,a)\tran \brk{\hatThetaKH -\thetaH + (\psiKH - \psiH)\hat{V}^{k}_{h+1}}
        \\
        &
        \qquad
        -
        \betaB \norm{\phiBarH^{\kEpoch}(x,a)}_{(\hatCovKH)^{-1}}
        \\
        &
        \le
        (\betaR + \betaP)\norm{\phiBarH^{\kEpoch}(x,a)}_{\covKH^{-1}}
        -
        \betaB \norm{\phiBarH^{\kEpoch}(x,a)}_{(\hatCovKH)^{-1}}
        \\
        &
        \le
        (\betaR + \betaP - \betaB)\norm{\phiBarH^{\kEpoch}(x,a)}_{(\hatCovKH)^{-1}}
        =
        0
        ,
    \end{align*}
    where the first relation is by definition of $\hat{Q}^{k}_h$ (\cref{eq:value-iteration-Q} in \cref{alg:r-opo-for-linear-mdp-regular-bonus}), the second relation is by \cref{eq:goodRewardPOLinearRegularBonus,eq:goodTransitionPOLinearRegularBonus} together with Cauchy-Schwarz, the third relation follows since $\hatCovKH \preceq \covKH$ and the last one  is by our choice $\betaB = \betaR + \betaP$ (see \cref{thm:regret-bound-PO-linear-regular-bonus} in \cref{appendix-sec:analysis} for hyper-parameter choices).
\end{proof}
Notice that the standard PO decomposition would have required that we bound the non-contracted expression
% the expectation with respect to $\piOpt$, i.e.,
$
    \EE[P, \piOpt]\brk[s]{\hat{Q}^{k}_h(x,a)  - \phi(x,a)\tran \brk{ \thetaKH + \psiH \hat{V}^{k}_{h+1} }}
    .
$
In \cite{sherman2023rate} the gap between this argument and that of \cref{lemma:optimism-PO-linear-regular-bonus} can be bounded using \cref{eq:sherman-cost-of-truncation}. However, the equivalent argument for our contraction is \cref{lemma:cost-of-contraction}, which is bounded only for $\piK$ and not for any policy $\pi \in \piClass$.
\begin{lemma}[Cost of optimism]
\label{lemma:cost-of-optimism-PO-linear-regular-bonus}
    Suppose that \cref{eq:goodRewardPOLinearRegularBonus,eq:goodTransitionPOLinearRegularBonus} hold, then for every $k \in [K]$
    \begin{align*}
        V_1^{k,\piK}(x_1)
        -
        \hat{V}^{k}_1(x_1)
        &
        \le
        3(\betaR + \betaP) \EE[P, \piK]\brk[s]*{
        \sum_{h \in [H]} 
        \norm{\phi(x_h,a_h)}_{(\covKH)^{-1}}}
        \\
        &
        \qquad
        +
        16H \betaWarmup^2 \EE[P, \piK]\brk[s]*{\sum_{h \in [H]}  \norm{\phi(x_h,a_h)}_{(\covKH)^{-1}}^2}
        +
        16 H^2 K^{-1}
        .
    \end{align*}
\end{lemma}

\begin{proof}
    First, by \cref{lemma:extended-value-difference} in \cref{sec:technical}, a value difference lemma by \cite{shani2020optimistic},
    \begin{align*}
        &
        V_1^{k,\piK}(x_1)
        -
        \hat{V}^{k}_1(x_1)
        % &
        =
        \EE[P, \piK]\brk[s]*{
        \sum_{h \in [H]}  \phi(x_h,a_h)\tran \brk*{ \thetaH  + \psiH\hat{V}_{h+1}^{k} } - \hat{Q}_k^{k}(x_h,a_h) }
        .
    \end{align*}
    Now, using \cref{lemma:cost-of-contraction} with $v = \thetaKH + \psiH \hat{V}_{h+1}^k$ we have that $\abs{\phi(x,a)\tran v} \le 4H$ (by \cref{eq:goodTransitionPOLinearRegularBonus}) and thus
    \begin{align*}
        \brk[s]{\phi(x_h,a_h) - \phiBarH^{\kEpoch}(x_h,a_h)}\tran \brk*{ \thetaH  + \psiH\hat{V}_{h+1}^{k} }
        \le
        16H \betaWarmup^2 \norm{\phi(x_h,a_h)}_{(\covKH)^{-1}}^2
        +
        16 H^2 K^{-1}
        .
    \end{align*}
    We can thus conclude the proof using standard arguments to show that
    \begin{align*}
        \phiBarH^{\kEpoch}(x_h,&a_h)\tran \brk*{ \thetaH  + \psiH\hat{V}_{h+1}^{k} } - \hat{Q}_k^{k}(x_h,a_h)
        \\
        \tag{\cref{eq:value-iteration-Q}}
        &
        =
        \phiBarH^{\kEpoch}(x_h,a_h)\tran \brk*{\thetaKH - \hatThetaKH + (\psiH - \psiKH)\hat{V}_{h+1}^{k} } 
        +
        \betaB \norm{\phiBarH^{\kEpoch}(x_h,a_h)}_{(\hatCovKH)^{-1}}
        \\
        \tag{Cauchy-Schwarz, \cref{eq:goodRewardPOLinearRegularBonus,eq:goodTransitionPOLinearRegularBonus}}
        &
        \le
        (\betaR + \betaP)
        \norm{\phiBarH^{\kEpoch}(x_h,a_h)}_{(\covKH)^{-1}}
        +
        \betaB \norm{\phiBarH^{\kEpoch}(x_h,a_h)}_{(\hatCovKH)^{-1}}
        \\
        \tag{$\det(\covKH) \le 2 \det(\hatCovKH), \betaB = \betaR + \betaP$}
        &
        \le
        3(\betaR + \betaP)
        \norm{\phiBarH^{\kEpoch}(x_h,a_h)}_{(\covKH)^{-1}}
        \\
        \tag{$\sig(x) \in [0,1], \forall x \in \RR$}
        &
        \le
        3(\betaR + \betaP)
        \norm{\phi(x_h,a_h)}_{(\covKH)^{-1}}
        ,
    \end{align*}
    as desired.
\end{proof}

\begin{ack}
    This project has received funding from the European Research Council (ERC) under the European Union’s Horizon 2020 research and innovation program (grant agreement No. 101078075).
    Views and opinions expressed are however those of the author(s) only and do not necessarily reflect those of the European Union or the European Research Council. Neither the European Union nor the granting authority can be held responsible for them. 
    This work received additional support from the Israel Science Foundation (ISF, grant number 2549/19), the Len Blavatnik and the Blavatnik Family Foundation, and the Israeli VATAT data science scholarship.
\end{ack}

\bibliography{bibliography}
\bibliographystyle{abbrvnat}

%%%%%%%%%%%%%%%%%%%%%%%%%%%%%%%%%%%%%%%%%%%%%%%%%%%%%%%%%%%%

\newpage
\appendix

\section{Analysis}
\label{appendix-sec:analysis}

We begin by defining a so-called ``good event'', followed by optimism, cost of optimism, and Policy Optimization cost. We conclude with the proof of \cref{thm:regret-bound-PO-linear-regular-bonus}.

\paragraph{Good event.}
We define the following good event $\Egood = \bigcap_{i=1}^3 E_i$, over which the regret is deterministically bounded:
\begin{flalign}
    \nonumber
    \tag{\cref{eq:goodRewardPOLinearRegularBonus}}
     &
     E_1 = \brk[c]*{\forall k \in [K], h \in [H] : \norm{\thetaKH - \hatThetaKH}_{\covKH} \le \betaR};
     &
     \\
     \nonumber
     \tag{\cref{eq:goodTransitionPOLinearRegularBonus}}
     &
     E_2 = \brk[c]*{k \in [K] , h \in [H] : \norm{(\psiH - \psiKH) \hat{V}_{h+1}^{k}}_{\covKH} \le \betaP, \norm{\hat{Q}_{h+1}^{k}}_\infty \le \betaQ};
     &
    \\
    \label{eq:goodBernsteinPOLinearRegularBonus}
    &
    E_3 
    =
    \brk[c]*{
    \sum_{k \in [K]} \EE[P, \piK] \brk[s]*{Y_k} \le \sum_{k \in [K]} 2 Y_k 
    +
    4 H (3(\betaR + \betaP) + 4 \betaQ \betaWarmup^2)
    \log \frac{6}{\delta}
    }.
    &
\end{flalign}
where
$
    Y_k
    =
        \sum_{h \in [H]} 
        3(\betaR + \betaP)
        \norm{\phi(x_h,a_h)}_{(\covKH)^{-1}}
        +
        4 \betaQ \betaWarmup^2 \norm{\phi(x_h,a_h)}_{(\covKH)^{-1}}^2
    .
$

\begin{lemma}[Good event]
\label{lemma:good-event-PO-linear-regular-bonus}
    Consider the parameter setting of \cref{thm:regret-bound-PO-linear-regular-bonus}.
    If $\etaO \le 1, \betaWarmup^2 \le K / (32 H d)$ then $\Pr[\Egood] \ge 1 - \delta.$
\end{lemma}
Proof in \cref{sec:good-event-proofs-PO-linear-regular-bonus}.

\paragraph{Policy online mirror descent.}
We use standard online mirror descent arguments to bound the local regret in each state.
\begin{lemma}[OMD]
\label{lemma:omd-term-PO-linear}
    Suppose that the good event $\Egood$ holds (\cref{eq:goodRewardPOLinearRegularBonus,eq:goodTransitionPOLinearRegularBonus,eq:goodBernsteinPOLinearRegularBonus}) and set $\etaO \le 1/\betaQ$, then 
    \begin{align*}
        \sum_{k \in \setEpochsE} \sum_{a \in \A} \hat{Q}_h^{k}(x,a) (\pi^\star_h(a \mid x) - \pi^{k}_h(a \mid x))
        &
        \le
        \frac{\log \abs{\A}}{\etaO} + \etaO \sum_{k \in \setEpochsE} \betaQ^2
        \quad
        ,
        \forall e \in \setEpochs, h \in [H], x \in \X
        .
    \end{align*}
\end{lemma}

\begin{proof}
    Notice that the policy $\pi^{k}$ is reset at the beginning of every epoch.
    Then, the lemma follows directly by \cref{lem:omd-regret} with $y_t(a) = - \hat{Q}_h^{k}(x,a) , x_t(a)  = \pi^{k}_h(a \mid x)$ and noting that $| \hat{Q}_h^{k}(x,a) | \le \betaQ$ by \cref{eq:goodTransitionPOLinearRegularBonus}.
\end{proof}

\paragraph{Epoch schedule.}
The algorithm operates in epochs.
At the beginning of each epoch, the policy is reset to be uniformly random.
We denote the total number of epochs by $\numEpochs$, the first episode within epoch $e$ by $\kEpoch$, and the set of episodes within epoch $e$ by $\setEpochsE$.
The following lemma bounds the number of epochs.
% (for proof see \cite{gao2021provably}).

\begin{lemma}
\label{lem:num-epochs}
    The number of epochs $E$ is bounded by $(3/2) d H \log (2K)$.
\end{lemma}

\begin{proof}
    Let $\mathcal{T}_h = \brk[c]{e_h^1, e_h^2, \ldots}$ be the epochs where the condition $\det(\covKH) \ge 2 \det(\hatCovKH)$ was triggered in \cref{line:repo-epoch-condition} of \cref{alg:r-opo-for-linear-mdp-regular-bonus}. Then we have that
    \begin{align*}
        \det(\Lambda_h^{\kEpoch})
        \ge
        \begin{cases}
            2 \det(\Lambda_h^{\kPrevEpoch})&, e \in \mathcal{T}_h
            \\
            \det(\Lambda_h^{\kPrevEpoch})&, \text{otherwise}
            .
        \end{cases}
    \end{align*}
    Unrolling this relation, we get that
    \begin{align*}
    \det(\Lambda_h^K)
    \ge
    2^{\abs{\mathcal{T}_h}-1} \det{I}
    =
    2^{\abs{\mathcal{T}_h}-1}
    ,
    \end{align*}
    and changing sides, and taking the logarithm we get that
    \begin{align*}
    \abs{\mathcal{T}_h}
    &
    \le
    1
    +
    \log_2 \det\brk*{\Lambda_h^K}
    \\
    \tag{$\det(A) \le \norm{A}^d$}
    &
    \le
    1
    +
    d \log_2 \norm{\Lambda_h^K}
    \\
    \tag{triangle inequality}
    &
    \le
    1
    +
    d \log_2 \brk*{1 + \sum_{k=1}^{K-1} \norm{\phi_h^k}^2}
    \\
    \tag{$\norm{\phi_h^k} \le 1$}
    &
    \le
    1
    +
    d \log_2 K
    \\
    &
    \le
    (3/2) d \log 2 K
    .
\end{align*}
We conclude that
\begin{align*}
    E
    =
    \abs{
    \brk*{
    \cup_{h \in [H]} {\mathcal{T}_h}
    }}
    \le
    \sum_{h \in [H]} \abs{\mathcal{T}_h}
    \le
    (3/2) d H \log (2K)
    .
    &
    \qedhere
\end{align*}
\end{proof}

\paragraph{Regret bound.}
\begin{theorem}
\label{thm:regret-bound-PO-linear-regular-bonus}
    Suppose that we run \cref{alg:r-opo-for-linear-mdp-regular-bonus} with parameters
    \begin{gather*}
        \etaO
        =
        \sqrt{\frac{3 d H \log (2K) \log \abs{\A}}{K \betaQ^2}}
        ,
        \betaB 
        =
        \betaR+\betaP
        ,
        \betaWarmup 
        =
        4(\betaR+\betaP)\log K
        ,
    \end{gather*}
    where
    $
        \betaR
        =
        2 \sqrt{2d \log (6KH/\delta)}
        ,
        \betaP
        =
        28 H d \sqrt{
        \log (10 K^5 H / \delta)
        }
        ,
        \betaQ
        =
        2 H
        .
    $
    % and the other parameters as in \cref{lemma:good-event-PO-linear-regular-bonus}.
    Then with probability at least $1 - \delta$ we incur regret at most
    \begin{align*}
        \regret
        &
        \le
        264 \sqrt{K d^3 H^4 \log(2K) \log (10 K^5 H / \delta)}
        +
        8 \sqrt{K d H^5 \log (2K) \log \abs{\A}}
        \\
        &
        \hspace{4em}
        +
        64 H^2 d \max\brk[c]{\betaWarmup^2, \log \abs{\A}} \log \frac{12 K}{\delta}
        \\
        &
        =
        O(
        \sqrt{K d^3 H^4 \log(K) \log (K H / \delta)}
        +
        \sqrt{K d H^5 \log (K) \log \abs{\A}}
        )
        .
    \end{align*}
    % i.e.,
    % $
    % \regret
    % \le
    % O\brk*{
    % \sqrt{d^5 H^9 K \log^9\frac{d H \abs{\A} K}{\delta}}
    % }
    % .
    % $
\end{theorem}
\begin{proof}
    First, if $\betaWarmup^2 > K / (32 H d)$ or $\eta \ge 1 / \betaQ$ then
    \begin{align*}
        \regret
        \le
        KH
        \le
        32 H^2 d \max\brk[c]{\betaWarmup^2, \log \abs{\A}} \log(2K)
        ,
    \end{align*}
    and the proof is concluded. Otherwise, if $\betaWarmup^2 \le K / (32 H d)$ then
    suppose that the good event $\Egood$ holds (\cref{eq:goodRewardPOLinearRegularBonus,eq:goodTransitionPOLinearRegularBonus,eq:goodBernsteinPOLinearRegularBonus}). By \cref{lemma:good-event-PO-linear-regular-bonus}, this holds with probability at least $1-\delta$.
    For any epoch $e \in [E]$, let $\setEpochsE$ be the set of episodes that it contains, and let $\bar{V}^{k,\pi}_1(x_1; \rho^{\kEpoch})$ denote the value of its contracted MDP as defined in \cref{sec:contracted-mdp} and \cref{eq:contracted-features} of \cref{alg:r-opo-for-linear-mdp-regular-bonus}.
We bound the regret as
\begin{align*}
    \regret
    & =
    \sum_{k \in [K]} V^{k,\pi^k}_1(x_1) - V^{k,\piOpt}_1(x_1)
    \\
    \tag{\cref{lemma:contracted-mdp-is-optimistic}}
    & \le
    \sum_{e \in \setEpochs}\sum_{k \in \setEpochsE} V^{k,\pi^k}_1(x_1) - \bar{V}^{k,\piOpt}_1(x_1; \rho^{\kEpoch})
    \\
    & 
    =
    \sum_{k \in [K]} V^{k,\pi^k}_1(x_1) - \hat{V}^{k}_1(x_1) 
    +
    \sum_{e \in \setEpochs} \sum_{k \in \setEpochsE} 
    \hat{V}^{k}_1(x_1) - \bar{V}^{k,\piOpt}_1(x_1; \rho^{\kEpoch})
    \\
    &
    =
    \underbrace{
    \sum_{k \in [K]} V^{k,\pi^k}_1(x_1) - \hat{V}^{k}_1(x_1)
    }_{(i) - \text{Bias / Cost of optimism}}
    \\
    &
    \quad
    +
    \underbrace{\sum_{e \in \setEpochs} \sum_{h \in [H]}  \EE[\pBarKe, \piOpt]\brk[s]*{ \sum_{k \in \setEpochsE} \sum_{a \in \A} \hat{Q}^{k}_h(x_h,a) \brk{ \pi^{k}_h(a \mid x_h) - \pi^\star_h(a \mid x_h) } } }_{(ii) - \text{OMD regret}}
    \\
    & \quad +
    \underbrace{\sum_{e \in \setEpochs} \sum_{k \in \setEpochsE} \sum_{h \in [H]} \EE[\pBarKe, \piOpt]\brk[s]*{ 
    \hat{Q}^{k}_h(x_h,a_h)  - \phiBarH^{\kEpoch}(x_h,a_h)\tran \brk{ \thetaKH + \psiH \hat{V}^{k}_{h+1} } } }_{(iii) - \text{Optimism}}
    ,
\end{align*}
where the last relation is by the extended value difference lemma (see \cite{shani2020optimistic} and \cref{lemma:extended-value-difference} in \cref{sec:technical}).
    For term $(i)$, we use \cref{lemma:cost-of-optimism-PO-linear-regular-bonus} as follows.
    \begin{align*}
        (i)
        & 
        \le
        \sum_{k \in [K]}
        \EE[P, \piK]\brk[s]*{
        \sum_{h \in [H]} 
        3(\betaR + \betaP)
        \norm{\phi(x_h,a_h)}_{(\covKH)^{-1}}
        +
        8 \betaQ \betaWarmup^2 \norm{\phi(x_h,a_h)}_{(\covKH)^{-1}}^2
        }
        +
        8 H \betaQ
        \\
        \tag{\cref{eq:goodBernsteinPOLinearRegularBonus}, $\betaWarmup \ge 120 (\betaR + \betaP)$}
        &
        \le 
        \sum_{k \in [K]}\brk[s]*{
        \sum_{h \in [H]} 
        6(\betaR + \betaP)
        \norm{\phi(x_h,a_h)}_{(\covKH)^{-1}}
        +
        16 \betaQ \betaWarmup^2 \norm{\phi(x_h,a_h)}_{(\covKH)^{-1}}^2
        }
        +
        20 H \betaQ \betaWarmup^2
        \log \frac{6}{\delta}
        \\
        \tag{\cref{lemma:elliptical-potential}}
        &
        \le 
        6(\betaR + \betaP) H \sqrt{2 K d \log(2K)}
        +
        32 \betaQ \betaWarmup^2 H d \log(2K)
        +
        20 H \betaQ \betaWarmup^2
        \log \frac{6}{\delta}
        \\
        &
        \le
        6(\betaR + \betaP) H \sqrt{2 K d \log(2K)}
        +
        32 H d \betaQ \betaWarmup^2 \log \frac{12 K}{\delta}
        .
    \end{align*}
   By \cref{lemma:omd-term-PO-linear,lem:num-epochs} (with our choice of $\etaO$) we have
    \begin{align*}
        (ii)
        \le
        \sum_{h \in [H]} \sum_{e \in \setEpochs} \EE[\pBarKe, \piOpt]\brk[s]*{ \frac{\log A}{\etaO} + \etaO \sum_{k \in \setEpochsE} \betaQ^2 }
        \le
        4 H \betaQ \sqrt{K d H \log (2K) \log \abs{\A}}
        .
    \end{align*}
    By \cref{lemma:optimism-PO-linear-regular-bonus} $(iii) \le 0$.
    Putting all bounds together, we get that
    \begin{align*}
        \regret
        &
        \le
        6(\betaR + \betaP) H \sqrt{2 K d \log(2K)}
        +
        32 H d \betaQ \betaWarmup^2 \log \frac{12 K}{\delta}
        +
        4 H \betaQ \sqrt{K d H \log (2K) \log \abs{\A}}
        \\
        &
        \le
        264 \sqrt{K d^3 H^4 \log(2K) \log (10 K^5 H / \delta)}
        +
        8 \sqrt{K d H^5 \log (2K) \log \abs{\A}}
        +
        64 H^2 d \betaWarmup^2 \log \frac{12 K}{\delta}
        \\
        &
        =
        O(
        \sqrt{K d^3 H^4 \log(K) \log (K H / \delta)}
        +
        \sqrt{K d H^5 \log (K) \log \abs{\A}}
        )
        .
        \qedhere
    \end{align*}
    % $
    % $
\end{proof}

\subsection{Proofs of good event}
\label{sec:good-event-proofs-PO-linear-regular-bonus}

We begin by defining function classes and properties necessary for the uniform convergence arguments over the value functions. We then proceed to define a proxy good event, whose high probability occurrence is straightforward to prove. We then show that the proxy event implies the desired good event.

\paragraph{Value and policy classes.}
We define the following class of restricted Q-functions:
\begin{align*}
    \widehat{\Q}(C_\beta, C_w, C_Q)
    =
    \brk[c]*{ \hat{Q}(\cdot,\cdot ; \beta, w, \Lambda,\mathcal{Z}) \mid \beta \in [0, C_\beta], \norm{w} \le C_w, (2K)^{-1}I \preceq \Lambda \preceq I, \norm{ \hat{Q}(\cdot,\cdot ; w, \Lambda,\mathcal{Z}) }_\infty \le C_Q }
    ,
\end{align*}
where $\hat{Q}(x,a ; \beta, w, \Lambda) = \brk[s]{w\tran \phi(x,a) - \beta \norm{\phi(x,a)}_{\Lambda}} \cdot \sig\brk{-\betaWarmup \norm{\phi(x,a)}_{\Lambda} + \log K}$.
Next, we define the following class of soft-max policies:
\begin{align*}
    \Pi (C_\beta, C_w)
    =
    \brk[c]*{ \pi(\cdot \mid \cdot; {\hat{Q}}) \mid \hat{Q} \in \widehat{\Q}(C_\beta, C_w, \infty)}
    ,
\end{align*}
where $\pi(a \mid x; {\hat{Q}}) = \frac{\exp \brk{\hat{Q}(x,a)}}{\sum_{a' \in \A} \exp \brk{\hat{Q}(x,a')}}$.
Finally, we define the following class of restricted value functions:
\begin{align}
\label{eq:value-function-class-def-regular-bonus}
    \widehat{\V}(C_\beta, C_w, C_Q)
    =
    \brk[c]*{ \hat{V}(\cdot ; \pi,\hat{Q}) \mid \pi \in \Pi(C_\beta K, C_w K, C_Q) , \hat{Q} \in \widehat{\Q}(C_\beta, C_w, C_Q)}
    ,
\end{align}
where $\hat{V}(x ; \pi,\hat{Q}) = \sum_{a \in \A} \pi(a \mid x) \hat{Q}(x,a)$.
The following lemma provides the bound on the covering number of the value function class defined above.
\begin{lemma}
\label{lemma:regular-bonus-function-class-covering-number}
    For any $\epsilon, C_w > 0, C_\beta, C_Q \ge 1$, we have 
    \begin{align*}
        \log \mathcal{N}_\epsilon \brk*{ \widehat{\V}(C_\beta, C_w, C_Q) }
        \le
        6 d^2 \log (1 + 4 (\sqrt{192K^3} C_Q C_\beta \betaWarmup)( K C_\beta +  K C_w +  \sqrt{d}) / \epsilon)
        ,
    \end{align*}
    where $\mathcal{N}_\epsilon$ is the covering number of a class in supremum distance.
\end{lemma}

\begin{proof}
    We begin by showing that the class of $Q$ function is Lipschitz in its parameters. For ease of notation, denote $y = \phi(x,a)$. Then
    \begin{align*}
        \norm{\nabla_\beta Q(x,a; \beta, w,\Lambda)}
        =
        \norm{y}_{\Lambda}
        \cdot \sig\brk{-\betaWarmup \norm{y}_{\Lambda} + \log K}
        \le
        1
        \tag{$\sig(\cdot) \in [0,1], \norm{y} \le 1, \Lambda \preceq I$}
    \end{align*}
    \begin{align*}
        \norm{\nabla_\theta \hat{Q}(x,a;\beta,w, \Lambda)}
        =
        \norm{y \cdot \sig\brk{-\betaWarmup \norm{y}_{\Lambda} + \log K}}
        \le
        1
        \tag{$\sig(\cdot) \in [0,1], \norm{y} \le 1$}
    \end{align*}
    \begin{align*}
        | Q(x,a;& \beta, w, \Lambda) - Q(x,a; \beta, w, \Lambda') |
        \\
        &
        \le
        \beta \abs{\norm{y}_\Lambda - \norm{y}_{\Lambda'}} 
        \cdot
        \sig\brk{-\betaWarmup \norm{y}_{\Lambda} + \log K}
        \\
        &
        \quad
        +
        \beta\norm{y}_{\Lambda'}
        \abs{
        \sig\brk{-\betaWarmup \norm{y}_{\Lambda} + \log K} - \sig\brk{-\betaWarmup \norm{y}_{\Lambda'} + \log K}
        }
        \\
        \tag{$\norm{\cdot}, \sig(\cdot)$ 1-Lipschitz, $\sig \in [0,1]$}
        &
        \le
        \beta \norm{(\Lambda^{1/2} - (\Lambda')^{1/2})y}
        +
        \beta \betaWarmup \norm{y}_{\Lambda'} \norm{(\Lambda^{1/2} - (\Lambda')^{1/2})y}
        \\
        \tag{$\norm{y} \le 1, \Lambda \preceq I, \betaWarmup \ge 1$}
        &
        \le
        2 \beta \betaWarmup \norm{\Lambda^{1/2} - (\Lambda')^{1/2}}
        \\
        \tag{\cref{lemma:norm-of-sqrt-is-lipschitz}, $\Lambda,\Lambda' \succeq (2K)^{-1}I$}
        &
        \le
        \sqrt{2K} \beta \betaWarmup \norm{\Lambda - \Lambda'}
        \\
        \tag{$\norm{\cdot} \le \norm{\cdot}_F$}
        &
        \le
        \sqrt{2K} \beta \betaWarmup \norm{\Lambda - \Lambda'}_F
        .
    \end{align*}
    We thus have that for any such $y$
    \begin{align*}
        % &
        | Q(x,a; \beta, w, &\Lambda) - Q(x,a; \beta', w', \Lambda') |
        \\
        &
        \le
        \abs{ Q(x,a; \beta, w, \Lambda) - Q(x,a; \beta', w, \Lambda)}
        +
        \abs{ Q(x,a; \beta', w, \Lambda) - Q(x,a; \beta', w', \Lambda)}
        \\
        &
        +
        \abs{ Q(x,a; \beta', w', \Lambda) - Q(x,a; \beta', w', \Lambda')}
        \\
        &
        \le
        \abs{\beta - \beta'} 
        +
        \norm{w - w'}
        +
        \sqrt{2K} \beta \betaWarmup \norm{\Lambda - \Lambda'}_F
        \\
        &
        \le
        \sqrt{3(\norm{w - w'}^2 + \abs{\beta - \beta'}^2 + (\sqrt{2K} \beta \betaWarmup)^2 \norm{\Lambda - \Lambda'}_F^2)}
        \\
        &
        \le
        \max\brk[c]{3,\sqrt{6K} \beta \betaWarmup}
        \sqrt{(\norm{w - w'}^2 + \abs{\beta - \beta'}^2 + \norm{\Lambda - \Lambda'}_F^2)}
        \\
        &
        =
        \max\brk[c]{3,\sqrt{6K} \beta \betaWarmup}
        \norm{(\beta, w, \Lambda) - (\beta', w', \Lambda')}
        ,
    \end{align*}
    where $(\beta,w,\Lambda)$ is a vectorization of the parameters. Assuming that $C_\beta \ge 1$, we conclude that $\widehat{\Q}(C_\beta, C_w, C_Q)$ is $\sqrt{6K} C_\beta \betaWarmup-$Lipschitz in supremum norm, i.e.,
    \begin{align*}
        \norm{\hat{Q}(\cdot, \cdot; \beta, w, \Lambda) - \hat{Q}'(\cdot, \cdot; \beta', w', \Lambda')}_\infty
        \le
        \sqrt{6K} C_\beta \betaWarmup
        \norm{(\beta, w, \Lambda) - (\beta', w', \Lambda')}
        .
    \end{align*}
    Next, notice that our policy class $\Pi(C_\beta K, C_w K)$ is a soft-max over the Q functions thus fitting Lemma 12 of \cite{sherman2023rate}. We conclude that the policy class is $\sqrt{24K^3} C_\beta \betaWarmup-$Lipschitz, in $\ell_1-$norm, i.e.,
    \begin{align*}
        \norm{
        \pi(\cdot \mid x; \beta, w, \Lambda)
        -
        \pi(\cdot \mid x; \beta', w', \Lambda')
        }_1
        \le
        \sqrt{24K^3} C_\beta \betaWarmup
        \norm{(\beta, w, \Lambda) - (\beta', w', \Lambda')}
        .
    \end{align*}
    Now, let $V, V' \in \widehat{\V}(C_\beta, C_w, C_Q)$ and $\theta=(\beta_1, w_1, \Lambda_1, \beta_2, w_2, \Lambda_2), \theta'=(\beta_1', w_1', \Lambda_1', \beta_2', w_2', \Lambda_2) \in \RR[2(1+d+d^2)]$ be their respective parameters. We have that for all $x \in \X$
    \begin{align*}
        \abs{V(x; \pi, \hat{Q}) - V(x; \pi', \hat{Q}')}
        \le
        \underbrace{
        \abs{V(x; \pi, \hat{Q}) - V(x; \pi, \hat{Q}')}
        }_{(i)}
        +
        \underbrace{
        \abs{V(x; \pi, \hat{Q}') - V(x; \pi', \hat{Q}')}
        }_{(ii)}
        .
    \end{align*}
    For the first term
    \begin{align*}
        (i)
        &
        =
        \abs*{\sum_{a \in \A} \pi(a\mid x) (\hat{Q}(x,a; \beta_2, w_2, \Lambda_2) - \hat{Q}(x,a; \beta_2', w_2', \Lambda_2'))}
        \\
        \tag{triangle inequality}
        &
        \le
        \sum_{a \in \A} \pi(a\mid x) \abs*{\hat{Q}(x,a; \beta_2, w_2, \Lambda_2) - \hat{Q}(x,a; \beta_2', w_2', \Lambda_2')}
        \\
        \tag{$\hat{Q}$ is $\sqrt{6K} C_\beta \betaWarmup$-Lipschitz, Cauchy-Schwarz}
        &
        \le
        \sqrt{6K} C_\beta \betaWarmup
        \norm{(\beta_2, w_2, \Lambda_2) - (\beta_2', w_2', \Lambda_2')}
        .
    \end{align*}
    For the second term
    \begin{align*}
        (ii)
        &
        =
        \abs*{\sum_{a \in \A} \hat{Q}'(x,a) (\pi(a \mid x) - \pi'(a \mid x))}
        \le
        C_Q \norm{\pi(\cdot \mid x) - \pi(\cdot \mid x)}_1
        \\
        &
        \le
        \sqrt{96K^3} C_Q C_\beta \betaWarmup
        \norm{(\beta_1, w_1, \Lambda_1) - (\beta_1', w_1', \Lambda_1')}
        ,
    \end{align*}
    where the first transition used that $\norm{Q}_\infty \le C_Q$ for all $Q \in \widehat{\Q}(C_\beta, C_w, C_Q)$ and the second used the Lipschitz property of the policy class shown above.
    Combining the terms and assuming that $C_Q \ge 1$ we get that
    \begin{align*}
        \abs{V(x; \pi, \hat{Q}) - V(x; \pi', \hat{Q}')}
        &
        \le
        \sqrt{96K^3} C_Q C_\beta \betaWarmup
        \norm{(\beta_1, w_1, \Lambda_1) - (\beta_1', w_1', \Lambda_1')}
        \\
        &
        \qquad
        +
        \sqrt{96K^3} C_Q C_\beta \betaWarmup
        \norm{(\beta_2, w_2, \Lambda_2) - (\beta_2', w_2', \Lambda_2')}
        \\
        &
        \le
        \sqrt{192K^3} C_Q C_\beta \betaWarmup
        \norm{\theta - \theta'}
        ,
    \end{align*}
    implying that $\widehat{\V}(C_\beta, C_w, C_Q)$ is $\sqrt{192K^3} C_Q C_\beta \betaWarmup-$Lipschitz in supremum norm.
    Finally, notice that
    \begin{align*}
        \norm{\theta} \le \abs{\beta_1} + \abs{\beta_2} + \norm{w_1} + \norm{w_2} + \norm{\Lambda_1}_F + \norm{\Lambda_2}_F
        \le
        2 K C_\beta + 2 K C_w + 2 \sqrt{d}
        ,
    \end{align*}
    and applying \cref{lemma:lipschitz-cover} concludes the proof.
\end{proof}

\paragraph{Proxy good event.}
We define a proxy good event $\EgoodBar = E_1 \cap \bar{E}_2 \cap E_3$ where
\begin{flalign}
    &
    \label{eq:goodTransitionProxPOLinearRegularBonus}
    \bar{E}_2 
    =
    \brk[c]*{k \in [K], h \in [H], V \in \widehat{\V}(\betaR+\betaP, 2 \betaQ K, \betaQh[h+1]) : \norm{(\psiH - \psiKH) V}_{\covKH} \le \betaP}
    ,
    &
\end{flalign}
where $\betaQh = 2(H+1-h), h \in [H+1]$.
Then we have the following result.
\begin{lemma}[Proxy good event]
\label{lemma:good-proxy-event-PO-regular-bonus}
    Consider the parameter setting of \cref{lemma:good-event-PO-linear-regular-bonus}. Then $\Pr[\EgoodBar] \ge 1 - \delta.$
\end{lemma}
\begin{proof}
    %
    % E_1
    %%%%%%%
    First, by \cref{lemma:reward-error} and our choice of parameters, $E_1$ (\cref{eq:goodRewardPOLinearRegularBonus}) holds with probability at least $1 - \delta/3$.
    %
    % \bar{E}_2
    %%%%%%%
    Next, applying \cref{lemma:dynamics-error-set-v,lemma:regular-bonus-function-class-covering-number}, we get that with probability at least $1-\delta / 3$ simultaneously for all $k \in [K], h \in [H], V \in \widehat{\V}(\betaR+\betaP, 2 \betaQ K, \betaQh[h+1])$
    \begin{align*}
        &
        \norm{(\psiH - \psiKH)V}_{\covKH}
        \\
        &
        \le
        4 \betaQh[h+1] \sqrt{
        d \log (2K) + 2\log (6H /\delta)
        +
        12 d^2 \log (1 + 8K (\sqrt{192K^3} C_\beta \betaWarmup)( K C_\beta +  K C_w +  1))
        }
        \\
        &
        \le
        4 \betaQ \sqrt{
        d \log (2K) + 2\log (6H /\delta)
        +
        12 d^2 \log (1 + 2K (\sqrt{192K^3} K / (32 H d))( \frac{1}{4} K \sqrt{K / (32 H d)} +  2 \betaQ K^2  +  1))
        }
        \\
        &
        \le
        4 \betaQ \sqrt{
        d \log (2K) + 2\log (6H /\delta)
        +
        12 d^2 \log (7 K^{9/2})
        }
        \\
        &
        \le
        4 \betaQ d \sqrt{
        12 \log (10 K^5 H / \delta)
        }
        \\
        &
        \le
        28 H d \sqrt{
        \log (10 K^5 H / \delta)
        }
        \\
        &
        =
        \betaP
        ,
    \end{align*}
    implying $\bar{E}_2$ (\cref{eq:goodTransitionProxPOLinearRegularBonus}).
    % 
    %
    % E_3
    %%%%%%%
    Finally, notice that 
    $
    \norm{\phi_h^k}_{(\covKH)^{-1}} \le 1
    ,
    $
    thus 
    $
    0
    \le
    Y_k
    \le
    H (3(\betaR + \betaP) + 4 \betaQ \betaWarmup^2)
    .
    $
    Using \cref{lemma:multiplicative-concentration}, a Bernstein-type inequality for martingales, we conclude that  $E_3$ (\cref{eq:goodBernsteinPOLinearRegularBonus}) holds with probability at least $1-\delta/3$.
\end{proof}

\paragraph{The good event.}
The following results show that the proxy good event implies the good event.

\begin{lemma}
\label{lemma:value-in-class-regular-bonus}
    Suppose that $\EgoodBar$ holds. If $\pi^{k}_h \in \Pi(K(\betaR+\betaP), 2 \betaQ K^2)$ for all $h \in [H]$ then $\hat{Q}^{k}_h \in \widehat{\Q}(\betaR + \betaP, 2 \betaQ K, \betaQh), \hat{V}^{k}_h \in \widehat{\V}(\betaR + \betaP, 2 \betaQ K, \betaQh)$ for all $h \in [H+1]$.
\end{lemma}
\begin{proof} 
    We show that the claim holds by backward induction on $h \in [H+1]$.
    \\
    \textbf{Base case $h = H+1$:}
    Since $\hat{V}_{H+1}^{k} = 0$ it is also implied that $\hat{Q}^{k}_{H+1} = 0$. Because $\beta, w = 0 \in \widehat{\Q}(\betaR + \betaP, 2 \betaQ K, \betaQh[H+1] = 0)$ we have that $\hat{Q}^{k}_{H+1} \in \widehat{\Q}(\betaR + \betaP, 2 \betaQ K, \betaQh[H+1] = 0)$, and similarly $V_{H+1}^{k} \in \widehat{\V}(\betaR + \betaP, 2 \betaQ K, \betaQh[H+1] = 0)$.
   
    \textbf{Induction step:}
    Now, suppose the claim holds for $h+1$ and we show it also holds for $h$. We have that
    \begin{align*}
        \abs{\hat{Q}_h^{k}(x,a)}
        &
        =
        \abs{
        \phiBarH^{\kEpoch}(x,a)\tran w_h^{k}
        -
        \betaB \norm{\bar{\phi}_h^{\kEpoch}(x,a)}_{(\hatCovKH)^{-1}}
        }
        \\
        &
        \le
        \abs{
        \phiBarH^{\kEpoch}(x,a)\tran (
        \thetaH + (\hatThetaKH - \thetaH)
        +
        (\psiKH - \psiH) \hat{V}^{k,i}_{h+1}
        +
        \psiH \hat{V}^{k,i}_{h+1}
        )
        }
        +
        \betaB \norm{\bar{\phi}_h^{\kEpoch}(x,a)}_{(\hatCovKH)^{-1}}
        \\
        \tag{triangle inequality, Cauchy-Schwarz, $\hatCovKH \preceq \covKH$}
        &
        \le
        1
        +
        \norm{\hat{V}_{h+1}^{k,i}}_\infty
        +
        \norm{\phiBarH^{\kEpoch}(x,a)}_{(\hatCovKH)^{-1}} 
        \brk[s]*{
        \norm{\hatThetaKH - \thetaH}_{\covKH}
        +
        \norm{(\psiKH - \psiH)\hat{V}^{k,i}_{h+1}}_{\covKH}
        +
        \betaB
        }
        \\
        \tag{induction hypothesis, \cref{eq:goodRewardPOLinearRegularBonus,eq:goodTransitionProxPOLinearRegularBonus}}
        &
        \le
        1
        +
        \betaQh[h+1]
        +
        \brk*{\betaR + \betaPh +\betaB}
        \norm{\phiBarH^{\kEpoch}(x,a)}_{(\hatCovKH)^{-1}}
        \\
        \tag{$\phiBarH^{\kEpoch}$ definition}
        &
        \le
        1
        +
        \betaQh[h+1]
        +
        \brk*{\betaR + \betaPh + \betaB}
        \max_{y \ge 0} \brk[s]{y \cdot \sig(-\betaWarmup y + \log K)}
        \\
        \tag{\cref{lemma:logistic-to-linear}}
        &
        \le
        1
        +
        \betaQh[h+1]
        +
        \frac{2 \log K}{\betaWarmup}
        \brk*{\betaR + \betaPh + \betaB}
        \\
        \tag{$\betaWarmup \ge 2 (\betaR + \betaPh + \betaB) \log K$}
        &
        \le
        2
        +
        \betaQh[h+1]
        \\
        &
        =
        \betaQh
        .
    \end{align*}
    Additionally, $\betaB = \betaR + \betaP$, $(\hatCovKH)^{-1} \preceq I$, $\norm{\hatCovKH} \le 1 + \sum_{k \in [K]} \norm{\phi_h^k} \le 2K$, thus $(\hatCovKH)^{-1} \succeq (2K)^{-1} I$, and
    \begin{align*}
        \norm{w_h^{k}}
        &
        =
        \norm{
        \hatThetaKH
        +
        \psiKH  \hat{V}^{k,i}_{h+1}
        }
        \le
        K + \betaQ K
        % \\
        % &
        \le
        2 \betaQ K
        =
        C_w
        .
    \end{align*}
    We conclude that $\hat{Q}_h^k \in \widehat{\Q}(\betaR + \betaP, 2 \betaQ K, \betaQh)$.
    Since $\pi_h^{k} \in \Pi(K(\betaR + \betaP), 2 \betaQ K^2)$, we also conclude that $\hat{V}^{k}_h \in \widehat{\V}(\betaR + \betaP, 2 \betaQ K, \betaQh)$, proving the induction step and finishing the proof.
\end{proof}

\begin{lemma*}[restatement of \cref{lemma:good-event-PO-linear-regular-bonus}]
    Consider the parameter setting of \cref{thm:regret-bound-PO-linear-regular-bonus}.
    If $\etaO \le 1, \betaWarmup^2 \le K / (32 H d)$ then $\Pr[\Egood] \ge 1 - \delta.$
\end{lemma*}

\begin{proof}
    Suppose that $\EgoodBar$ holds. By \cref{lemma:good-proxy-event-PO-regular-bonus}, this occurs with probability at least $1-\delta$. We show that $\EgoodBar$ implies $\Egood$, thus concluding the proof.
    Notice that 
    \begin{align*}
        &
        \pi_h^{k}(a | x)
        % &
        \propto
        \exp\brk*{\eta \sum_{k' = \kEpoch}^{k-1} \hat{Q}_h^{k'}(x,a)}
        \\
        &
        =
        \exp\brk*{
         \sig\brk{-\betaWarmup\norm{\phi(x,a)}_{(\hatCovKH)^{-1}} + \log K}
         \cdot
         \brk[s]*{
         \phi(x,a)\tran  \sum_{k' = \kEpoch}^{k-1} \eta w_h^{k}
         -
         \eta \betaB (k - \kEpoch)\norm{\phi(x,a)}_{(\hatCovKH)^{-1}}
         }
         }
        .
    \end{align*}
    We show by induction on $k \in \setEpochsE$ that $\pi_h^{k} \in \Pi(K(\betaR+\betaP), 2\betaQ K^2)$ for all $h \in [H]$. For the base case, $k = \kEpoch$, $\pi_h^{k}$ are uniform, corresponding to $w,\beta=0 \in \Pi(K(\betaR+\betaP), 2 \betaQ K^2)$.
    Now, suppose the claim holds for all $k' < k$. Then by \cref{lemma:value-in-class-regular-bonus} we have that $\hat{Q}_h^{k'} 
    \in \widehat{\Q}(\betaR + \betaP, 2\betaQ K, \betaQh)$ for all $k' < k$ and $h \in [H]$. This implies that 
    $
    \norm{\sum_{k' = \kEpoch}^{k-1} \eta w_h^{k}}
    \le
    2 \betaQ K^2
    $
    for all $h \in [H]$, thus $\pi_h^{k} \in \Pi(K(\betaR+\betaP), 2\betaQ K^2)$ for all $h \in [H]$, concluding the induction step.

    Now, since $\pi_h^{k} \in \Pi(K(\betaR+\betaP), 2\betaQ K^2)$ for all $k \in [K], h \in [H]$, we can apply \cref{lemma:value-in-class-regular-bonus} to get that $\hat{Q}_h^{k} \in \widehat{\Q}(\betaR+\betaP, 2 \betaQ K, \betaQh), \hat{V}_h^{k} \in \widehat{\V}(\betaR+\betaP, 2 \betaQ K, \betaQh)$ for all $k \in [K], h \in [H]$.
    Using $\bar{E}_2$ (\cref{eq:goodTransitionProxPOLinearRegularBonus}) we conclude that $E_2$ (\cref{eq:goodTransitionPOLinearRegularBonus}) holds, thus concluding the proof.
\end{proof}

\newpage
\section{Technical tools}
\label{sec:technical}

\subsection{Online Mirror Descent}
We begin with a standard regret bound for entropy regularized online mirror descent (hedge). See \cite[Lemma 25]{sherman2023rate}.

\begin{lemma}
\label{lem:omd-regret}
    Let $y_1, \dots, y_T \in \RR[A]$ be any sequence of vectors, and $\eta > 0$ such that $\eta y_t(a) \ge -1$ for all $t \in [T],a \in [A]$. Then if $x_t \in \Delta_A$ is given by $x_1(a) = 1/A \ \forall a$, and for $t \ge 1$:
    \begin{align*}
        x_{t+1}(a)
        =
        \frac{x_t(a) e^{- \eta y_t(a)}}{\sum_{a' \in [A]} x_t(a') e^{- \eta y_t(a')}},
    \end{align*}
    then,
    \begin{align*}
        \max_{x \in \Delta_A} \sum_{t=1}^T \sum_{a \in [A]} y_t(a) \brk*{x_t(a) - x(a)}
        \le
        \frac{\log A}{\eta}
        +
        \eta \sum_{t=1}^T \sum_{a \in [A]} x_t(a) y_t(a)^2.
    \end{align*}
\end{lemma}

\subsection{Value difference lemma}
% We have three value difference lemmas (see, e.g., \cite{shani2020optimistic}).
% One is for comparing two different MDPs, while the other is for comparing to different policies in the same MDP.
We use the following extended value difference lemma by \cite{shani2020optimistic}. We note that the lemma holds unchanged even for MDP-like structures where the transition kernel $P$ is a sub-stochastic transition kernel, i.e., one with non-negative values that sum to at most one (instead of exactly one).

\begin{lemma}[Extended Value difference Lemma 1 in \cite{shani2020optimistic}]
\label{lemma:extended-value-difference}
    Let $\M$ be a (sub) MDP, $\pi, \hat{\pi} \in \piClass$ be two policies, $\hat{Q}_h: \X \times \A \to \RR, h \in [H]$ be arbitrary function, and $\hat{V}_h : \X \to \RR$ be defined as
    $\hat{V}(x) = \sum_{a \in \A} \hat{\pi}_h(a \mid x) \hat{Q}_h(x, a).$ Then
    \begin{align*}
        V^\pi_1(x_1) - \hat{V}_1(x_1)
        &
        =
        \EE[P,\pi]
        \brk[s]*{
        \sum_{h \in [H]}\sum_{a \in \A} \hat{Q}_h(x_h,a) (\pi(a \mid x_h) - \hat{\pi}(a \mid x_h))
        }
        \\
        &
        +
        \EE[P,\pi]
        \brk[s]*{
        \sum_{h \in [H]} \ell_h(x_h, a_h) + \sum_{x' \in \X} P(x' \mid x_h, a_h) \hat{V}_{h+1}(x') - \hat{Q}_h(x_h, a_h)
        }
        .
    \end{align*}
    We note that, in the context of linear MDP $\ell_h(x_h, a_h) + \sum_{x' \in \X} P(x' \mid x_h, a_h) \hat{V}_{h+1}(x') = \phi(x_h, a_h)\tran(\thetaH + \psiH \hat{V}_{h+1}).$
\end{lemma}

\subsection{Algebraic lemmas}

Next, is a well-known bound on harmonic sums \citep[see, e.g.,][Lemma 13]{cohen2019learning}. This is used to show that the optimistic and true losses are close on the realized predictions. 
% See proof below for completeness.
\begin{lemma}
\label{lemma:elliptical-potential}
    Let $z_t \in \RR[d']$ be a sequence such that $\norm{z_t}^2 \le \lambda$, and define $V_t = \lambda I + \sum_{s=1}^{t-1} z_s z_s\tran$. Then
    \begin{align*}
        \sum_{t=1}^{T} \norm{z_t}_{V_t^{-1}}
        \le
        \sqrt{T \sum_{t=1}^{T} \norm{z_t}_{V_t^{-1}}^{2}}
        \le
        \sqrt{2 T d' \log (T+1)}
        .
    \end{align*}
\end{lemma}
%
% \begin{proof}%[of \cref{lemma:harmonicBound}]
%     Notice that $0 \le z_t\tran V_t^{-1} z_t \le \norm{z_t}^2 / \lambda \le 1$.
%     Next, notice that 
%     \begin{align*}
%         \det(V_{t+1})
%         =
%         \det(V_t + z_t z_t\tran)
%         =
%         \det(V_t)\det(I + V_t^{-1/2} z_t z_t\tran V_t^{-1/2})
%         =
%         \det(V_t) (1 + z_t\tran V_t^{-1} z_t)
%         ,
%     \end{align*}
%     which follows from the matrix determinant lemma. We thus have that
%     \begin{align*}
%         \tag{$x \le 2 \log (1+x), \forall x \in [0,1]$}
%         z_t\tran V_t^{-1} z_t
%         \le
%         \log (1 + z_t\tran V_t^{-1} z_t)
%         =
%         \log \frac{\det(V_{t+1})}{\det(V_t)}
%     \end{align*}
%     We conclude that
%     \begin{align*}
%         \sum_{t=1}^{T} z_t\tran V_t^{-1} z_t
%         &
%         \le
%         2 \sum_{t=1}^{T} \log \brk*{\det(V_{t+1}) / \det(V_t)}
%         \\
%         \tag{telescoping sum}
%         &
%         =
%         2 \log \brk*{\det(V_{T+1}) / \det(\lambda I)}
%         \\
%         \tag{$\det(V) \le \norm{V}^{d'}$}
%         &
%         \le
%         2 d' \log (\norm{V_{T+1}} / \lambda)
%         \\
%         \tag{triangle inequality}
%         &
%         \le
%         2 d' \log \brk*{1 + \sum_{t=1}^{T} \norm{z_t}^2 / \lambda}
%         \\
%         &
%         \le
%         2 d' \log (T+1)
%         .
%     \end{align*}
%     The proof is concluded by applying the Cauchy-Schwarz inequality to get that
%     $\sum_{t=1}^{T} \norm{z_t}_{V_t^{-1}}
%         \le
%         \sqrt{T \sum_{t=1}^{T} \norm{z_t}_{V_t^{-1}}^{2}}
%         .
%     $
% \end{proof}

Next, we need the following well-known matrix inequality.
\begin{lemma}[\cite{cohen2019learning}, Lemma 27]
\label{lemma:matrix-norm-inequality}
    If $N \succeq M \succ 0$ then for any vector $v$
    \begin{align*}
        \norm{v}_{N}^2
        \le
        \frac{\det{N}}{\det{M}} \norm{v}_{M}^2
    \end{align*}
\end{lemma}

Next, we need a bound on the Lipschitz constant of the spectral norm of a square-root matrix.
% https://math.stackexchange.com/questions/1838034/lipschitz-continuity-of-sqrta
\begin{lemma}
\label{lemma:norm-of-sqrt-is-lipschitz}
For any $\lambda > 0$ and matrices $\Lambda, \Lambda' \in \RR[d \times d]$ satisfying $\Lambda, \Lambda' \succeq \lambda I$ we have that
\begin{align*}
    \norm{\Lambda^{1/2} - \Lambda'^{1/2}}
    \le
    \frac{1}{2 \sqrt{\lambda}}\norm{\Lambda - \Lambda'}
    .
\end{align*}
\end{lemma}
\begin{proof}
    Let $\mu$ be an eigenvalue of $\Lambda^{1/2} - \Lambda'^{1/2}$ with eigenvector $x \in \RR[d]$. Then we have that
    \begin{align*}
        \abs{x\tran (\Lambda - \Lambda') x}
        &
        =
        \abs{
        x\tran (\Lambda^{1/2} - \Lambda'^{1/2}) \Lambda^{1/2} x
        +
        x\tran \Lambda'^{1/2}(\Lambda^{1/2} - \Lambda'^{1/2}) x
        }
        \\
        &
        =
        \abs{\mu} x\tran (\Lambda^{1/2} + \Lambda'^{1/2}) x
        .
    \end{align*}
    Next, notice that
    $
    \abs{x\tran (\Lambda - \Lambda') x}
    \le
    \norm{x}^2 \norm{\Lambda - \Lambda'}
    ,
    $
    and
    $
    x\tran (\Lambda^{1/2} + \Lambda'^{1/2})
    \ge
    2 \sqrt{\lambda}\norm{x}^2
    .
    $
    We thus therefore change sides to get that
    \begin{align*}
        \abs{\mu}
        \le
        \frac{1}{2\sqrt{\lambda}}\norm{\Lambda - \Lambda'}
        ,
    \end{align*}
    and since we can take $\mu = \pm \norm{\Lambda^{1/2} - \Lambda'^{1/2}}$, the proof is concluded.
\end{proof}

Finally, we need the following bounds on the logistic function.
\begin{lemma}
    \label{lemma:logistic-to-linear}
    For any $K \ge 1, \beta > 0$ we have that
    \begin{align*}
        \max_{y \ge 0} \brk[s]{y \cdot \sig(-\beta y + \log K)}
        \le
        \frac{2 \log K}{\beta}
    \end{align*}
\end{lemma}
\begin{proof}
    First, if $y' \le (2\log K)/\beta$ then using $\sig(y) \in [0,1]$ we have that
    \begin{align*}
        y' \sig(-\beta y' + \log K)
        \le
        y'
        \le 
        (2 \log K) / \beta
        ,
    \end{align*}
    as desired.
    Now, if $y' \ge (2 \log K) / \beta$ then
    \begin{align*}
        y' \sig(-\beta y' + \log K)
        \le
        y' \sig(-\beta y' / 2)
        =
        \frac{y'}{1 + e^{\beta y' / 2}}
        \le
        \frac{y'}{\beta y' / 2}
        =
        \frac{2}{\beta}
        ,
    \end{align*}
    where the first inequality also used that $\sig(y)$ is increasing and the last inequality used that $1 + e^y \ge y$ for all $y \ge 0$.
\end{proof}

\begin{lemma}\label{lemma:logistic-to-quadratic}
    For any $K \ge 1, z \ge 0$ we have that $\sig(z - \log K) \le 2 (z^2 + K^{-1}).$
\end{lemma}
\begin{proof}
    Recall the logistic function $\sig(z) = 1 / (1+ e^{-x})$ and define the function $g(z) = \sig(z - \log K) - (z + K^{-1/2})^2$. We show that $g(z) \le 0$ for all $z \ge 0$.
    First, notice that 
    \begin{align*}
        g(0) = \sig(-\log K) - K^{-1} = (K+1)^{-1} - K^{-1} \le 0
        .
    \end{align*}
    Next, recall that $\sig'(x) = \sig(x) \sig(-x)$ and thus
    \begin{align*}
        g'(z) = \sig(z - \log K) \sig(-z + \log K) - 2 (z + K^{-1/2})
        .
    \end{align*}
    Examining $z=0$ we further have that
    \begin{align*}
        g'(0)
        &
        =
        \sig(-\log K) \sig(\log K) - 2 K^{-1/2}
        \\
        &
        =
        (K+1)^{-1} (1 + K^{-1})^{-1} - 2 K^{-1/2}
        \\
        &
        \le
        2 \brk[s]{(K+1)^{-1} - K^{-1/2}}
        \le
        0
        ,
    \end{align*}
    where the last two inequalities used $K \ge 1$.
    Now, we have that
    \begin{align*}
        g''(z)
        =
        \sig(z - \log K) \sig(-z + \log K)^2
        -
        \sig(z - \log K)^2 \sig(-z + \log K) 
        -
        2
        \le
        0
        ,
    \end{align*}
    where the inequality is since $\sig(z) \in [0,1]$ for all $z \in \RR$.
    Since $g(0), g'(0) \le 0$ and $g''(z) \le 0$ for all $z \ge 0$, we conclude that $g(z) \le 0$ for all $z \ge 0$. The proof is concluded using the AM-GM inequality.
\end{proof}

\subsection{Concentration bounds}

We give the following Bernstein type tail bound (see e.g., \cite[Lemma D.4]{rosenberg2020near}.
\begin{lemma}
\label{lemma:multiplicative-concentration}
Let $\brk[c]{X_t}_{t \ge 1}$
be a sequence of random variables with expectation adapted to a filtration
$\mathcal{F}_t$.
Suppose that $0 \le X_t \le 1$ almost surely. Then with probability at least $1-\delta$
\begin{align*}
    \sum_{t=1}^{T} \EE \brk[s]{X_t \mid \mathcal{F}_{t-1}}
    \le
    2 \sum_{t=1}^{T} X_t
    +
    4 \log \frac{2}{\delta}
\end{align*}
\end{lemma}

% Next, we state a standard concentration inequality for self-normalized processes.
% \begin{lemma}[Concentration of Self-Normalized Processes \cite{abbasi2011improved}]
% \label{lemma:self-normalized}
%     Let $\eta_t$ ($t \ge 1$) be a real-valued stochastic process with corresponding filtration $\mathcal{F}_t$. Suppose that $\eta_t \mid \mathcal{F}_{t-1}$ are zero-mean $R$-subGaussian, and let $\phi_t$ ($t \ge 1$) be an $\RR[d]$-valued, $\mathcal{F}_{t-1}$-measurable stochastic process. Assume that $\Lambda_0$ is a $d \times d$ positive definite matrix and let $\Lambda_t = \Lambda_0 + \sum_{s=1}^{t} \phi_s \phi_s\tran$. Then for any $\delta > 0$, with probability at least $1 - \delta$, we have for all $t \ge 0$
%     \begin{align*}
%         \norm*{
%         \sum_{s=1}^{t} \phi_s \eta_s
%         }_{\Lambda_t^{-1}}^2
%         \le
%         2 R^2 \log \brk[s]*{\frac{\det\brk{\Lambda_t}^{1/2} \det\brk{\Lambda_0}^{-1/2}}{\delta}}
%     \end{align*}
%     Additionally, if $\Lambda_0 = \lambda I$ and $\norm{\phi_t}^2 \le \lambda$, for all $t \ge 1$ then
%     \begin{align*}
%         \norm*{
%         \sum_{s=1}^{t} \phi_s \eta_s
%         }_{\Lambda_t^{-1}}^2
%         \le
%         R^2 \brk[s]{d \log (t+1) + 2\log (1/\delta)}
%     \end{align*}
% \end{lemma}

% \subsection{Reward and dynamics estimation bounds}

We state the well-known self normalized error bounds for regularized least squares estimation of the rewards and dynamics (see e.g., \cite{abbasi2011improved,jin2020provably}).
% guarantee for the rewards least-squares estimate.
\begin{lemma}[reward error bound]
\label{lemma:reward-error}
    Let $\hatThetaKH$ be as in \cref{eq:LS-estimate-theta} of \cref{alg:r-opo-for-linear-mdp-regular-bonus}. With probability at least $1 - \delta$, for all $k \ge 1, h \in [H]$
    \begin{align*}
        \norm{\thetaH - \hatThetaKH}_{\covKH}
        \le
        2 \sqrt{2d \log (2KH/\delta)}
        .
    \end{align*}
\end{lemma}
\begin{lemma}[dynamics error uniform convergence]
\label{lemma:dynamics-error-set-v}
    Let $\psiKH : \RR[\X] \to \RR[d]$ be the linear operator defined in \cref{eq:LS-estimate-psi} inside \cref{alg:r-opo-for-linear-mdp-regular-bonus}.
    For all $h \in [H]$, let $\mathcal{V}_h \subseteq \RR[\X]$ be a set of mappings $V: \X \to \RR$ such that $\norm{V}_\infty \le \beta$ and $\beta \ge 1$. 
    % \begin{align*}
    %     \psiKH V
    %     =
    %     (\covKH)^{-1} \sum_{\tau=1}^{k-1} \phi^\tau_h V(x_{h+1}^\tau)
    %     .
    % \end{align*}
    With probability at least $1-\delta$, for all $h \in [H]$, $V \in \mathcal{V}_{h+1}$ and $k \ge 1$
    \begin{align*}
        \norm{(\psiH - \psiKH)V}_{\covKH}
        \le
        4 \beta \sqrt{d \log (K+1) + 2\log (H \mathcal{N}_\epsilon /\delta)}
        ,
    \end{align*} 
    where $\epsilon \le \beta \sqrt{d}/{2K}$, $\mathcal{N}_\epsilon = \sum_{h \in [H]} \mathcal{N}_{h,\epsilon}$, and $\mathcal{N}_{h,\epsilon}$ is the $\epsilon-$covering number of $\mathcal{V}_h$ with respect to the supremum distance.
\end{lemma}

\subsection{Covering numbers}

The following results are (mostly) standard bounds on the covering number of several function classes.

\begin{lemma}
\label{lemma:ball-cover}
    For any $\epsilon > 0$, the $\epsilon$-covering of the Euclidean ball in $\RR[d]$ with radius $R \ge 0$ is upper bounded by $(1 + 2R/\epsilon)^d$.
\end{lemma}

% \begin{lemma}[Lemma D.6 of \cite{jin2020provably}]
% \label{lemma:covering-number}
%     Let $\mathcal{V}$ denote a class of functions $V : \mathcal{X} \to \RR$ with the following parametric form
%     \begin{align*}
%         V(\cdot)
%         =
%         \clip_R\brk[s]*{\max_a \theta\tran \phi(\cdot,a) + \beta \sqrt{\phi(\cdot, a)\tran \Lambda^{-1} \phi(\cdot, a)}}
%         ,
%     \end{align*}
%     where $R \ge 0$ is a constant and $(\theta,\beta,\Lambda, R)$ are parameters satisfying $\norm{\theta} \le L, \beta \in [0, B], \lambda_{\min}(\Lambda) \ge \lambda$ where $\lambda_{\min}$ denotes the minimal eigenvalue. Assume $\norm{\phi(x,a)} \le 1$ for all $(x,a)$ pairs, and let $\mathcal{N}_\epsilon$ be the $\epsilon-$covering number of $\mathcal{V}$ with respect to the supremum distance. Then
%     \begin{align*}
%         \log \mathcal{N}_\epsilon
%         \le
%         d \log(1 + 4 L / \epsilon)
%         +
%         d^2 \log \brk[s]{1 + 8 \sqrt{d} B^2 / (\lambda \epsilon^2)}
%         .
%     \end{align*}
% \end{lemma}

\begin{lemma}
\label{lemma:lipschitz-cover}
    Let $\mathcal{V} = \brk[c]{V(\cdot; \theta) : \norm{\theta} \le W}$ denote a class of functions $V : \mathcal{X} \to \RR$. Suppose that any $V \in \mathcal{V}$ is $L$-Lipschitz with respect to $\theta$ and supremum distance, i.e.,
    \begin{align*}
        \norm{V(\cdot; \theta_1) - V(\cdot; \theta_2)}_\infty
        \le
        L \norm{\theta_1 - \theta_2}
        ,
        \quad
        \norm{\theta_1}, \norm{\theta_2} \le W
        .
    \end{align*}
    Let $\mathcal{N}_\epsilon$ be the $\epsilon-$covering number of $\mathcal{V}$ with respect to the supremum distance. Then
    \begin{align*}
        \log \mathcal{N}_\epsilon
        \le
        d \log(1 + 2 W L / \epsilon)
    \end{align*}
\end{lemma}
\begin{proof}
    Let $\Theta_{\epsilon/L}$ be an $(\epsilon/L)$-covering of the Euclidean ball in $\RR[d]$ with radius $W$. Define $\mathcal{V}_\epsilon = \brk[c]{V(\cdot; \theta) : \theta \in \Theta_{\epsilon / L}}$.
    By \cref{lemma:ball-cover} we have that $\log\abs{\mathcal{V}_\epsilon} \le d \log(1 + 2 W L / \epsilon)$. We show that $\mathcal{V}_\epsilon$ is an $\epsilon$-cover of $\mathcal{V}_{\epsilon}$, thus concluding the proof.
    Let $V \in \mathcal{V}$ and $\theta$ be its associated parameter. Let $\theta' \in \Theta_{\epsilon / L}$ be the point in the cover nearest to $\theta$ and $V' \in \mathcal{V}$ its associated function. Then we have that 
    \begin{align*}
        \norm{V(\cdot) - V'(\cdot)}_\infty
        =
        \norm{V(\cdot; \theta) - V(\cdot; \theta')}_\infty
        \le
        L \norm{\theta - \theta'}
        \le
        L (\epsilon / L) 
        = 
        \epsilon
        .
        &
        \qedhere
    \end{align*}
\end{proof}

\end{document}